\documentclass{article} % For LaTeX2e
\usepackage{iclr2024_conference,times}

\usepackage{amsmath,amsfonts,bm}
\def\1{\bm{1}}
\DeclareMathAlphabet{\mathsfit}{\encodingdefault}{\sfdefault}{m}{sl}
\SetMathAlphabet{\mathsfit}{bold}{\encodingdefault}{\sfdefault}{bx}{n}

\iclrfinalcopy
\usepackage{hyperref}
 \hypersetup{
	colorlinks=true,
	linkcolor=orange,
	filecolor=magenta,      
	urlcolor=orange,
	citecolor=orange,
}
\usepackage{url}

\usepackage{url}            % simple URL typesetting
\usepackage{booktabs}       % professional-quality tables
\usepackage{amsfonts}       % blackboard math symbols
\usepackage{nicefrac}       % compact symbols for 1/2, etc.
\usepackage{microtype}      % microtypography

\usepackage{graphicx}
\usepackage{lscape}
\usepackage{rotating}
\usepackage[ruled, vlined, linesnumbered]{algorithm2e}
\usepackage{algorithmic}
\usepackage[utf8]{inputenc}
\usepackage{float}
\usepackage{url}            % simple URL typesetting
\usepackage{xspace}
\usepackage{amsthm, amsmath, mathtools, amsfonts, amssymb, dsfont, enumitem, mathabx}
\usepackage{bm, bbm}
\usepackage[mathscr]{euscript}
\usepackage{adjustbox}
\usepackage{footmisc}
\usepackage[capitalize,noabbrev]{cleveref}
\usepackage{multirow}
\usepackage{transparent}
\usepackage{caption}
\usepackage{subcaption}
\usepackage{tcolorbox}

\usepackage{wrapfig}
\usepackage{enumitem}
\usepackage[font=small]{caption}
\usepackage{bbm}
\usepackage{pifont}
\usepackage{cancel}

\usepackage[ruled, vlined, linesnumbered]{algorithm2e}
\usepackage{algorithmic}
\usepackage{minitoc}
\usepackage{thmtools, thm-restate}

\makeatletter
\def\thm@space@setup{%
  \thm@preskip=2pt
  \thm@postskip=2pt % or whatever, if you don't want them to be equal
}
\makeatother

\theoremstyle{plain}
\theoremstyle{definition}

\theoremstyle{remark}

\definecolor{lightorange}{HTML}{ff7f2a}
\definecolor{lighterorange}{HTML}{ffe6d5}
\newtcolorbox{summarybox}{colback=lighterorange,colframe=lightorange}

\newcommand{\E}[2]{\mathbb{E}_{#1}{\left[#2\right]}}

\newcommand{\f}[2]{D_{f}(#1\ ||\ #2)}

\newcommand{\set}[1]{\left\{#1\right\}}

\newcommand{\para}[1]{\textbf{#1}}

\newcommand{\D}{\mathcal{D}}

\renewcommand{\S}{\mathcal{S}}
\newcommand{\A}{\mathcal{A}}

\renewcommand{\hat}{\widehat}
\renewcommand{\S}{\mathcal{S}}

\newcommand{\Mix}{\texttt{Mix}}
\newcommand{\dmix}{\Mix_\beta(d, \rho)(s, a, g)}
\newcommand{\demix}{\Mix_\beta(q, \rho)(s, a, g)}

\title{Score Models for Offline\\Goal-Conditioned Reinforcement Learning}

\author{Harshit Sikchi\thanks{Work done partially during an internship at Meta AI. Correspondence to hsikchi@utexas.edu}\\
University of Texas at Austin\\
\And
Rohan Chitnis\thanks{Equal Contribution}\\
Meta AI \\
\And
 Ahmed Touati\footnotemark[2]\\
Meta AI \\
\And
Alborz Geramifard \\
Meta AI \\
\And
 Amy Zhang \\
University of Texas at Austin, Meta AI\\
\And
 Scott Niekum \\
UMass Amherst\\
}

\newcommand{\reb}[1]{\textcolor{black}{#1}}
\newcommand{\scogo}{\texttt{SMORe}\xspace}

\usepackage{xcolor} 
\definecolor{mydarkblue}{rgb}{0,0.08,0.45}
\definecolor{myredorange}{rgb}{0.8,0.33,0}

\BeforeBeginEnvironment{wrapfigure}{\setlength{\intextsep}{1pt}}
\everypar{\looseness=-1}
\iclrfinalcopy % Uncomment for camera-ready version, but NOT for submission.
\begin{document}

\maketitle

\begin{abstract}
Offline Goal-Conditioned Reinforcement Learning (GCRL) is tasked with learning to achieve multiple goals in an environment purely from offline datasets using sparse reward functions. Offline GCRL is pivotal for developing generalist agents capable of leveraging pre-existing datasets to learn diverse and reusable skills without hand-engineering reward functions. However, contemporary approaches to GCRL based on supervised learning and contrastive learning are often suboptimal in the offline setting. An alternative perspective on GCRL optimizes for occupancy matching, but necessitates learning a discriminator, which subsequently serves as a pseudo-reward for downstream RL. Inaccuracies in the learned discriminator can cascade, negatively influencing the resulting policy. We present a novel approach to GCRL under a new lens of mixture-distribution matching, leading to our discriminator-free method: SMORe. The key insight is combining the occupancy matching perspective of GCRL with a convex dual formulation to derive a learning objective that can better leverage suboptimal offline data. SMORe learns \textit{scores} or unnormalized densities representing the importance of taking an action at a state for reaching a particular goal. SMORe is principled and our extensive experiments on the fully offline GCRL benchmark composed of robot manipulation and locomotion tasks, including high-dimensional observations, show that SMORe can outperform state-of-the-art baselines by a significant margin. \\

\centering{\textbf{Project page (Code and Videos):} \href{https://hari-sikchi.github.io/smore/}{\color{myredorange}hari-sikchi.github.io/smore/}}
\end{abstract}

\section{Introduction}

A generalist agent will require a vast repertoire of skills, and large amounts of offline pre-collected data offer a way to learn useful skills without any environmental interaction.  Many subfields of machine learning like vision and NLP have enjoyed great success by designing objectives to learn a general model from large and diverse datasets. In robot learning, offline interaction data has become more prominent in the recent past~\citep{ebert2021bridge}, with the scale of the datasets growing consistently~\citep{walke2023bridgedata, padalkar2023open}. Goal-conditioned reinforcement learning (GCRL) offers a principled way to acquire a variety of useful skills without the prohibitively difficult process of hand-engineering reward functions. In GCRL, the agent learns a policy to accomplish a variety of goals in the environment. The rewards are sparse and goal-conditioned: 1 when the agent's state is proximal to the goal and 0 otherwise. However, the benefit of not requiring the designer to hand-engineer dense reward functions can also be a curse, because learning from sparse rewards is difficult. Driving progress in fundamental offline GCRL algorithms thus becomes an important aspect of moving towards performant generalist agents whose skills scale with data.

Despite recent progress in developing methods for goal-reaching in the online setting (where environment interactions are allowed), a number of these methods are either suboptimal in the offline setting or suffer from learning difficulties. Prior GCRL algorithms can largely be classified into one of three categories: iterated behavior cloning, RL with sparse rewards, and contrastive learning. Iterated behavior cloning or goal-conditioned supervised learning approaches~\citep{ghosh2019learning,yang2019imitation} have been shown to be provably suboptimal~\citep{eysenbach2022imitating} for GCRL. Modifying single-task RL methods~\citep{silver2014deterministic,kostrikov2021offline} for GCRL with 0-1 reward implies learning a $Q$-function that predicts the discounted probability of goal reaching, which makes it essentially a density model. Modeling density directly is a hard problem, an insight which has prompted the development of methods~\citep{eysenbach2020c} that learn density-ratio instead of densities, as classification is an easier problem than density estimation. Contrastive RL approaches to GCRL~\citep{eysenbach2020c,eysenbach2022contrastive,zheng2023stabilizing} aim to do precisely this and are the main methods to enjoy success for applying GCRL in high-dimensional observation spaces. However, when dealing with offline datasets, contrastive RL approaches~\citep{eysenbach2022contrastive,zheng2023stabilizing} are suboptimal, as they only learn a policy that is a greedy improvement over the Q-function of the data generation policy. This begs the question: \emph{How can we derive a performant GCRL method that learns near-optimal policies from offline datasets of suboptimal quality?} 
% To address this, we resort to the insight that GCRL can be formulated as an occupancy-matching problem that can be efficiently solved in an off-policy manner using the tools of convex duality.

In this work, we leverage the underexplored insight of formulating GCRL as an occupancy matching problem. Occupancy matching between the joint state-action-goal visitation distribution induced by the current policy and the distribution over state-actions that transition to goals can be shown to be equivalent to optimizing a max-entropy GCRL objective. Occupancy matching has been studied extensively in imitation learning~\citep{ghasemipour2020divergence} and often requires learning a discriminator and using the learned discriminator for downstream policy learning through RL. Indeed, a prior GCRL work~\citep{ma2022far}  explores a similar insight. Unfortunately, errors in learned discriminators can compound and adversely affect the learned policy's performance, especially in the offline setting where these errors cannot be corrected with further interaction with the environment.

% between joint state-action-goal visitation 
% of current policy and uniform distribution over (state, action, goal) such that the taking the action at the state leads to the goal in one timestep.

Going beyond the shortcomings of the previous methods, our proposed method combines the insight of formulating GCRL as an occupancy matching problem along with an efficient, discriminator-free dual formulation that learns from offline data. The resulting algorithm \scogo forgoes learning density functions or classifiers, but instead learns unnormalized densities or \textit{scores} that allow it to produce near-optimal goal-reaching policies. The scores are learned via a Bellman-regularized contrastive procedure that makes our method a desirable candidate for GCRL with high-dimensional observations, avoiding the need for density modeling. Our experiments represent a wide variety of goal-reaching environments -- consisting of robotic arms, anthropomorphic hands, and locomotion environments. We lay out the following contributions: 1) on the extended offline GCRL benchmark, our results demonstrate that \scogo significantly outperforms prior methods in the offline GCRL setting. 2) In line with our hypothesis, discriminator-free training makes \scogo particularly robust to decreasing goal-coverage in the offline dataset, a property we demonstrate in the experiments. 3) We test \scogo for zero-shot GCRL on a prior benchmark~\citep{zheng2023stabilizing} for high dimensional vision-based GCRL where contrastive RL approaches are the only class of GCRL methods that have been successful, and show improved performance over other state-of-the-art baselines.

\section{Problem Formulation}
\reb{We consider an infinite horizon discounted Markov Decision Process denoted by the tuple $\mathcal{M} = (\mathcal{S}, \mathcal{A}, p, r, \gamma, d_0)$, where $\S$ is the state space, $\A$ is the action space, $p$ is the transition probability function, $r: \S \times \A \rightarrow \mathbb{R} $ is the reward function, $\gamma \in (0, 1)$ is the discount factor, and $d_0$ is the initial state distribution. We constrain ourselves to the goal-conditioned RL setting, where we additionally assume a goal space $\mathcal{G}$ where states in $\mathcal{S}$ are mapped to the goal space using a known mapping:  $\phi:\mathcal{S} \rightarrow \mathcal{G}$. The reward function $r(s,a,g)$ in GCRL is sparse and also depends on the goal. A goal conditioned policy $\pi:\mathcal{S}\times \mathcal{G} \rightarrow \Delta(\mathcal{A})$ outputs a distribution over actions in a given state conditioned on a goal. Given a distribution over desired evaluation goals $q^{\texttt{test}}(g)$, the objective of goal-conditioned RL is to find a policy $\pi_g$\footnote{We use the subscript g to make the policy's conditioning on g explicit.} that maximizes the expected discounted return:}
% \textbf{Goal-Conditioned Reinforcement Learning:} We consider an infinite-horizon Markov decision process (MDP)~\citep{puterman2014markov} $\mathcal{M}=(\mathcal{S}, \mathcal{A},r,p, d_0, \gamma)$ with state space $\mathcal{S}$, action space $\mathcal{A}$, deterministic rewards $r(s,a)$, transition probabilities $p(s' \mid s,a)$  from state $s$ to $s'$ given action $a$, initial state distribution $d_0(s)$, and discount factor $\gamma \in (0, 1)$. A policy $\pi:\mathcal{S} \rightarrow \Delta(\mathcal{A})$ outputs a distribution over actions in a given state. 
% In goal-conditioned RL, the MDP additionally assumes a goal space $\mathcal{G} \coloneqq \{\phi(s) \mid s \in \mathcal{S}\}$, where the state-to-goal mapping $\phi:\mathcal{S} \rightarrow \mathcal{G}$ is known. The sparse reward function $r(s,a,g)$ as well as the policy $\pi(a\mid s,g)$ depend on the commanded goal $g \in \mathcal{G}$. Given a distribution over desired evaluation goals $q^{\texttt{test}}(g)$, the objective of goal-conditioned RL is to find a policy $\pi_g$\footnote{We use the subscript g to make the policy's conditioning on g explicit.} that maximizes the expected discounted return:
\vspace{-0.1cm}
% \hs{rewrite}
\begin{equation}
\label{eq:gcrl-objective}
    J(\pi_g) \coloneqq \mathbb{E}_{g\sim q^{\texttt{test}}((g)),s_0\sim d_0, a_t \sim \pi_g}\left[\sum_{t=0}^{\infty} \gamma^t r(s_t,a_t, g)\right].
\end{equation}
% \vspace{-0.2cm}
We denote by $P^{\pi_g}$ the transition operator induced by the policy $\pi_g$ defined as $P^{\pi_g}S(s, a, g) \coloneqq \E{s' \sim p(\cdot | s, a), a' \sim \pi_g(\cdot | s',g) }{ S(s',a',g) }$ , for any \textit{score} function $S:\mathcal{S} \times \mathcal{A} \times \mathcal{G}  \rightarrow \mathbb{R}$. \reb{We use $d^\pi(s,a\mid g)$ to denote the discounted goal-conditioned state-action occupancy distribution of $\pi_g$, i.e $d^{\pi_g}(s,a\mid g)=(1-\gamma)\pi(a|s,g)\sum_{t=1}^{\infty}[\gamma^t \text{Pr}(s_t=s|\pi_g, d_0)]$.}
% The goal-conditioned discounted state-action occupancy distribution $d^\pi(s,a\mid g)$ of $\pi_g$ is given by:
% \begin{equation}
% \label{eq:pi-occupancies}
% \begin{split}
% d^{\pi_g}(s,a\mid g) \coloneqq (1-\gamma) \sum_{t=0}^{\infty} \gamma^t \mathrm{Pr}(s_t=s, a_t=a \mid s_0 \sim d_0, a_t \sim {\pi_g}(\cdot \mid s_t,g), s_{t+1} \sim p(\cdot \mid s_t,a_t)),
% \end{split}
% \end{equation}
which represents the expected discounted time spent in each state-action pair by the policy $\pi_g$ conditioned on the goal $g$. 
% It follows that $\pi_g(a\mid s,g) = \frac{d^{\pi_g}(s,a\mid g)}{d^{\pi_g}(s\mid g)}$, where $d^{\pi_g}(s\mid g) \coloneqq \sum_{a\in \mathcal{A}} d^{\pi_g}(s, a \mid g)$.
For complete generality, in GCRL, the distribution of goals the policy is trained on often differs from the test goal distribution. To make this distinction clear we define the training distribution $q^{\texttt{train}}(g)$, a uniform measure over goals we desire to learn to optimally reach during training. We write $d^{\pi_g}(s,a,g) = q^{\texttt{train}}(g)d^{\pi_g}(s,a\mid g)$ as the joint state-action-goal visitation distribution of the policy $\pi_g$ under the training goal distribution. A state-action-goal occupancy distribution must satisfy the \textit{Bellman flow constraint} in order for it to be a valid occupancy\footnote{We will use ``occupancy'' and ``visitation'' interchangeably.} distribution for some stationary policy $\pi_g$, $\forall s \in \mathcal{S},a \in \mathcal{A}, g\in \mathcal{G}$: 
\begin{align}
\label{eq:bellman-flow-constraint}
 d(s,a,g) & = (1-\gamma) d_0(s,g)\pi_g(a\mid s,g) + \gamma \sum_{s',a'}p(s\mid s', a') d(s', a', g) \pi_g(a\mid s,g),~~ 
% & = (1-\gamma) d_0(s,g)\pi_g(a\mid s,g) + \gamma (P^{\pi_g})^\top d(s', a', g),
\end{align}
where $d_0(s,g)=d_0(s) q^{\texttt{train}}(g)$. Finally, given $d^{\pi_g}$, we can express the learning objective for the GCRL agent under the training goal distribution as $J^{\texttt{train}}(\pi_g) = \frac{1}{1-\gamma} \mathbb{E}_{(s,a,g) \sim d^{\pi_g}}[r(s,a,g)]$. 
% \todoah{the expectation should be over the goal conditional ditribution $d^{\pi_g}(s, a \mid g)$ and not the joint.} 
% \todoah{Do we have goals defined along the offline dataset in the experimental setting ? }

\reb{In this work, we focus on the offline setup where the agent cannot interact with the environment $\mathcal{M}$ and instead has access to a offline dataset of $\mathcal{D} \coloneqq \{\tau_i\}_{i=1}^N$, where each trajectory $\tau^{(i)} = (s_0^{(i)},a_0^{(i)}, r_0^{(i)}, s_1^{(i)},...; g^{(i)})$ with $s_0^{(i)} \sim d_0$.} The trajectories are usually relabelled with the $q^{\texttt{train}}(g)$ during learning. We denote the joint state-action-goal distribution of the offline dataset $\mathcal{D}$ as $\rho(s,a,g)$. 

\section{ Score-models for Offline Goal Conditioned Reinforcement Learning}
\label{sec:method}

In this section, we introduce our method in two parts: First, we build up the equivalence of the GCRL objective to the occupancy matching problem in Section~\ref{sec:gcrl_occupancy}, and then we derive a discriminator-free dual objective for solving the occupancy matching problem using off-policy data in Section~\ref{sec:scogo_method}. 
% In Section~\ref{sec:scogo_contrastive_connection}, we explore a connection between \scogo and contrastive learning. 
Finally, we present the algorithm for \scogo under practical considerations in Section~\ref{sec:scogo_algorithm}.

\subsection{GCRL as an occupancy matching problem}
\label{sec:gcrl_occupancy}

 Define a \textit{goal-transition distribution} $q(s,a,g)$ in a stochastic MDP as $q(s,a,g) ~\propto ~q^{\texttt{train}}(g)\E{s' \sim p(\cdot \mid s, a)}{\mathbb{I}_{\phi(s')=g}}$. Intuitively, the distribution has probability mass on each transition that leads to a goal.
 % pairs that transition to training distribution over goals.
% Define a \textit{goal-transition distribution} $q(s,a,g)$ with uniform measure over states $s$ s.t $p(s,a)=g$ and 0 otherwise. \todoah{it is weird to write $p(s,a)=g$ since p is distribution and g is vector. I think we want to say that 
% $q(s,a,g) = \frac{p(g)}{|\S \times \A|}\E{s' \sim p(\cdot \mid s, a)}{\mathbb{I}_{\phi(s')=g}}$. Would you like that the next state is equal to g, or the current state is equal to g? } Intuitively this distribution has full probability mass at states that correspond to goal states and zero otherwise. 
We formulate the GCRL problem as an occupancy matching problem by searching for the policy $\pi_g$ that minimizes the discrepancy between its state-action-goal occupancy distribution and the goal-transition distribution $q(s, a, g)$:
\begin{equation}
    \label{eq:multi-task-imitation}
   \texttt{ Occupancy matching problem:}~~ \D_f(d^{\pi_g}(s,a,g) \| q(s,a,g)),
\end{equation}
where $D_f$ denotes an $f$-divergence with generator function $f$.
Note that the $q$ distribution is potentially unachievable by any goal-conditioned policy $\pi_g$. Firstly, it does not account for the initial transient phase that the policy must navigate to reach the desired goal. Secondly, even if we consider only the stationary regime (when $\gamma \rightarrow 1$), it may not be dynamically possible for the policy to continuously remain at the goal and rather necessitate cycling around the goal.
However, in Proposition~\ref{lemma:gcrl-imitation}, we show that the occupancy matching in Eq.~\ref{eq:multi-task-imitation} offers a principled objective since it forms a lower bound to the max-entropy GCRL problem.

\begin{restatable}[]{proposition}{gcrlimitationf}
\label{lemma:gcrl-imitation}
Consider a stochastic MDP, a stochastic policy $\pi$, and a sparse reward function $r(s,a,g)=\E{s'\sim p(\cdot|s,a)}{\mathbb{I}(\phi(s')=g,q^{\texttt{train}}(g)>0)}$ where $\mathbb{I}$ is an indicator function. Define a soft goal transition distribution to be $q(s,a,g)~\propto ~\text{exp}(\alpha ~r(s,a,g))$. The following bounds hold for any $f$-divergence that upper bounds KL-divergence (eg. $\chi^2$, Jensen-Shannon):
\begin{equation}
      J^{train}(\pi_g) + \frac{1}{\alpha}\mathcal{H}(d^{\pi_g}) \ge - \frac{1}{\alpha}\D_f(d^{\pi_g}(s,a,g) \| q(s,a,g))+C,
\end{equation}
where $\mathcal{H}$ denotes the entropy, $\alpha$ is a temperature parameter and $C$ is the partition function for $e^{R(s,a,g)}$. Furthermore, the bound is tight when $f$ is the KL-divergence.
\end{restatable}

\reb{\cite{ma2022far} (in Proposition 4.1) presented a similar result connecting state-goal distribution matching ( i.e $D_{KL}(d^{\pi}(s,g)||q(s,g))$) to GCRL objective and Proposition~\ref{lemma:gcrl-imitation} extends their results to goal-transition distribution matching. Matching action-free distributions necessitates constructing a loose lower bound that is tractable to optimize. By considering goal-transition distributions we sidestep constructing a loose lower bound and instead directly obtain a tractable distribution matching objective~\citep{ghasemipour2020divergence, kostrikov2019imitation} that is tight under KL-divergence.}

\begin{figure*}
\begin{center}
          \includegraphics[width=0.80\linewidth]{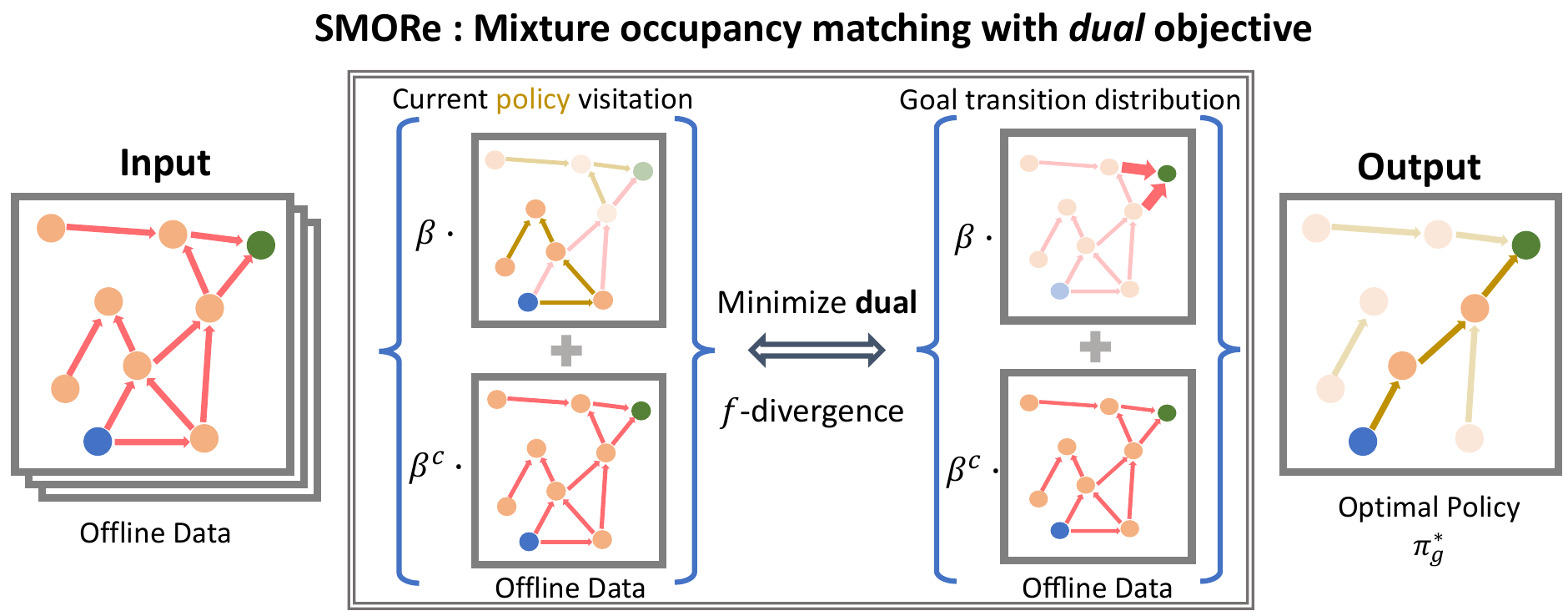}
\end{center}
\vspace{-3.0mm}
\caption{Illustration of the \scogo objective where $\beta^c=1-\beta$: \scogo matches a mixture distribution of current policy and offline data to a mixture of the goal-transition distribution and offline data in order to find the optimal goal reaching policy.}
\label{fig:main_fig}
\vspace{-0.5cm}
\end{figure*}
How does converting a GCRL objective to an imitation learning objective make learning easier? Estimating the $f$-divergence still requires estimating the joint policy visitation probabilities $d^{\pi_g}(s,a,g)$, which itself presents a challenging problem. We show in the following section that we can leverage convex duality to transform the imitation learning problem into an off-policy optimization problem, removing the need to sample from $d^{\pi_g}(s,a,g)$ whilst being able to leverage offline data collected from arbitrary sources.

\subsection{SMORe: A Dual Formulation for Occupancy Matching}
\label{sec:scogo_method}
The previous section establishes GCRL as an occupancy matching problem (Eq.~\ref{eq:multi-task-imitation}) but provides no way to use offline data whose joint visitation distribution is given by $\rho(s,a,g)$. To leverage offline data to learn performant goal-reaching policies, we consider a surrogate objective to the occupancy matching learning problem by matching \textit{mixture} distributions:
\begin{equation}
\label{eq:matching_mixture_distributions}
    \min_{\pi_g} \D_f(\Mix_\beta(d^{\pi_g}, \rho)  \| \Mix_\beta(q, \rho) ),
\end{equation}
where for any two distributions $\mu_1$ and $\mu_2$, $\Mix_\beta(\mu_1, \mu_2)$ denotes the mixture distribution with coefficient $\beta\in (0,1]$ defined as $\Mix_\beta(\mu_1, \mu_2) = \beta \mu_1 + (1-\beta) \mu_2$. Proposition~\ref{lemma:gcrl-imitation-mixture} (in appendix) shows the matching mixture distribution\footnote{Note that Eq.~\ref{eq:matching_mixture_distributions} shares the same global optima as the previous occupancy matching objective at $d^\pi_g(s,a,g)=q(s,a,g)$ when $q$ is an achievable visitation under some policy and recovers the original objective in Eq.~\ref{eq:multi-task-imitation} when $\beta=1$.} provably maximizes a lower bound to the Lagrangian relaxation of the max-entropy GCRL objective subject to a dataset regularization constraint.
% This implies that matching mixture distributions effectively solves the GCRL problem irrespective of the offline data distribution $d^O$ and hyperparameter $\beta$.  
% Let $d_\text{mix}^O :=\beta d^{\pi_g}(s,a,g) + (1-\beta) d^O(s,a,g)$ and $d_\text{mix}^{Q,O} := \beta q(s,a,g)+(1-\beta)d^O(s,a,g)$.
We can rewrite the mixture occupancy matching objective as a convex program with linear constraints~\citep{manne1960linear,nachum2020reinforcement}:
\begin{align}
\label{eq:primal_gcrl_f_mixture}
 &~~~~~~~~~~~~~~~~~~~~~~~~~~~~~~~~~~~~\max_{\pi_g, d} - \D_f(\Mix_\beta(d, \rho)  \| \Mix_\beta(q, \rho)) \nonumber\\
    &~~\text{s.t}~\textstyle d(s,a,g)=(1-\gamma)d_0(s,g)\pi(a|s)+\gamma \sum_{s' \in \S } d(s',a',g) p(s|s',a')\pi(a'|s',g), \; \forall s \in \S  .
\end{align}
% \todoah{Do we need the nonegativity constraint here ? isn't it a by-product of Bellman flow ? the mazimization problem look independent from the optimizaiton variable $d$ because of the notation of $d_{\text{mix}}^O$ that hides $d$. shall we reconsider the notation to make the dependancy more explicit ? } \hs{happy for suggestions to change the notation. i am not too satisfied with it either but couldnt think of a better one. we can remove non-neg.}
An illustration of this objective can be found in \cref{fig:main_fig}.
Effectively, we have simply rewritten Eq.~\ref{eq:matching_mixture_distributions} into an equivalent problem by considering an arbitrary probability distribution $d(s,a,g)$ in the optimization objective, only to later constrain it to be a valid probability distribution induced by some policy $\pi_g$ using the \textit{Bellman-flow constraints}. The motivation behind this construction of the primal form is that we have made computing the Lagrangian-dual easier as this objective is convex with linear constraints. Theorem~\ref{lemma:dual_gcrl} shows that we can leverage tools from convex duality to obtain an unconstrained dual problem that does not require computing $d^{\pi_g}(s,a,g)$ or sampling from it, while effectively leveraging offline data. 

\begin{restatable}[]{theorem}{scogoQ}
    \label{lemma:dual_gcrl}
    The dual problem to the primal occupancy matching objective (Equation ~\ref{eq:primal_gcrl_f_mixture}) is given by:
    \begin{align}
    \label{eq:scogo}
        \max_{\pi_g}\min_{S}  \beta (1-\gamma)& \E{d_0,\pi_g}{S(s,a,g)}  +\E{\Mix_\beta(q, \rho)}{f^*(\gamma P^{\pi_g} S(s,a,g )-S(s,a,g))}  \\
        & -(1-\beta) \E{\rho}{\gamma P^{\pi_g} S(s,a,g)-S(s,a,g)}, \nonumber
    \end{align}
where $f^*$ is conjugate function of $f$ and $S$ is the Lagrange dual variable defined as $S:\mathcal{S}\times\mathcal{A}\times\mathcal{G}\to \mathbb{R}.$ Moreover, as strong duality holds from Slater's conditions the primal and dual share the same optimal solution $\pi^*_g$ for any offline transition distribution $\rho$.
\end{restatable}

To our knowledge, the closest prior works to our proposed method are GoFAR~\citep{ma2022far} and Dual-RL~\citep{sikchi2023dual}. GoFAR considers the special case of KL-divergence for the imitation formulation and derives a dual objective that requires learning the density ratio $\frac{\rho(s,g)}{q(s,g)}$ in the form of a discriminator and using this as a pseudo-reward. This leads to compounding errors in the downstream RL optimization when learning the density ratio is challenging, e.g. in the case of low coverage between $\rho(s,a,g)$ and $q(s,a,g)$. We show this phenomenon experimentally in Section~\ref{sec:scogo_robustness}. Dual-RL~\citep{sikchi2023dual} uses convex duality for matching visitation distribution of realizable expert demonstrations and does not deal with the GCRL setting. \emph{Our contribution is a novel method for GCRL that is discriminator-free, applicable for a number of $f$-divergences, and robust to low coverage of goals in the offline dataset.}

\textbf{Sampling from the goal-transition distribution: } Goal relabelling is an effective technique to address reward sparsity by widening the training goal distribution $q^{\texttt{train}}(g)$. It utilizes knowledge about reaching other goals, possibly unrelated to test goals, to help in reaching the test distribution of goals $q^\texttt{{test}}(g)$.  In the most general case, $q^{\texttt{train}}(g)$ can be set to a uniform distribution over goals corresponding to all the states in the offline data. A common method, Hindsight Experience Replay (HER)~\citep{andrychowicz2017hindsight} chooses a training goal distribution that depends on the current sampled state from the offline dataset as well as the data-collecting policies. In this setting, the sampling distribution used for training Eq~\ref{eq:scogo}, $\rho(s,a,g)$, can no longer be factorized into $\rho(s,a)$ and $q^{\texttt{train}}(g)$, as goals are conditionally dependent on state-actions. However, our formulation can naturally account for learning from such relabelled data as the \scogo objective in Eq~\ref{eq:scogo} is derived considering the joint distribution $\rho(s,a,g)$. In this setting, we construct our goal transition distribution $q(s,a,g)$ as the uniform distribution over all transitions that lead to the goals selected by the HER procedure --- in practice, this amounts to first selecting $g$ through HER and then selecting $\set{s,a}$ that transitions to the selected goal from the offline dataset to get a sample $\set{s,a,g}$ from goal transition distribution. We emphasize that relabelling does not change the test distribution of goals, which is an immutable property of the environment.

\subsection{Practical Algorithm}
\label{sec:scogo_algorithm}

% \subsection{Connection to contrastive learning}
% \label{sec:scogo_contrastive_connection}
To devise a stable learning algorithm we consider the Pearson $\chi^2$ divergence. Pearson $\chi^2$ divergence has been found to lead to distribution matching objectives that are stable to train as a result of a smooth quadratic generator function $f$~\citep{garg2021iq,al2023ls,sikchi2023dual}.  Our dual formulation \scogo simplifies to the following objective:
\begin{multline}
\label{eq:contrastive_scogo}
\resizebox{0.92\textwidth}{!}{
   $\max_{\pi_g}\min_{S}  \color{purple}\overbrace{\color{black} \beta (1-\gamma)\E{(s, g) \sim d_0, a \sim \pi_g(\cdot \mid s, g)}{S(s, a,g)} +\beta \gamma \E{(s,a,g) \sim q,s'\sim p(\cdot|s,a), a' \sim \pi_g(\cdot \mid s', g)} { S(s',a',g )}}^{\text{Decrease score at transitions under current policy $\pi_g$}}$}\\
        \resizebox{0.92\textwidth}{!}{$-\color{teal}\underbrace{\color{black}\beta \E{(s,a,g) \sim q}{S(s,a,g)}}_{\text{Increase score at the proposed goal transition distribution}}\color{black}+0.25\color{brown}\underbrace{\color{black}\E{(s,a,g) \sim \Mix_{\beta}(q, \rho)}{( \gamma S(s',\pi_g(s'),g )-S(s,a,g))^2}}_{\text{Smoothness/Bellman regularization}}.$}
\end{multline}
\begin{wrapfigure}{r}{0.45\textwidth}
\centering
\begin{minipage}[t]{.45\textwidth}
\begin{algorithm}[H]
    \small
    \caption{SMORe}
    \label{alg:ALG1}
    \begin{algorithmic}[1]
    \STATE Init $S_\phi$, $M_\psi$, and $\pi_\theta$ \STATE Params: expectile $\tau$, mixture ratio $\beta$, temperature $\alpha$
    \STATE Let $\mathcal{D}=\hat{\rho} = \{(s, a, s', g)\}$ be an offline dataset and $q$ be goal-transition distribution
    \FOR{$t=1..T$ iterations}
        \STATE Train $S_\phi$ via Eq.~\ref{eq:practical_s_update}
        \STATE Train $M_\psi$ via  Eq.~\ref{eq:practical_m_update} 
        \STATE Update $\pi_\theta$ via Eq.~\ref{eq:practical_pi_update} 
    \ENDFOR
    \end{algorithmic}
\end{algorithm}
\end{minipage}
% \vskip-10pt
\end{wrapfigure}
Equation~\ref{eq:contrastive_scogo} suggests a contrastive procedure, maximizing the score at the goal-transition distribution and minimizing the score at the offline data distribution under the current policy with Bellman regularization. The  Bellman regularization has the interpretation of discouraging neighboring $S$ values from deviating far and smoothing the score landscape. Instantiating with KL divergence results in an objective with similar intuition while resembling an InfoNCE~\citep{oord2018representation} objective. Although Propositions~\ref{lemma:gcrl-imitation} and~\ref{lemma:gcrl-imitation-mixture} suggest that KL divergence gives an objective that is a tighter bound to the GCRL objective, prior work has found KL divergence to be unstable in practice~\citep{sikchi2023dual,garg2023extreme} for dual optimization. 

% Eq.~\ref{eq:contrastive_scogo} further suggests the additional Bellman regularization term should be 0.25 relative to the weighting for the contrastive losses in order to devise a principled GCRL method. 
It is important to note that $S$-function is not grounded to any rewards and does not serve as a probability density of reaching goals, but is rather a score function learned via \textit{a Bellman-regularized contrastive learning procedure}.

% \newpage
% \subsection{Practical Algorithm}

We now derive a practical approach for \scogo in the offline GCRL setting. We use parameterized functions: $S_\phi(s,a,g)$,  $M_\psi(s,g)$,  $\pi_\theta(a|s,g)$.  The offline learning regime necessitates measures to constrain the learning policy to the offline data support in order to prevent overestimation due to maximizing $\pi_g$ in Eq.~\ref{eq:contrastive_scogo} over potentially out-of-distribution actions. Inspired by prior work~\citep{kostrikov2021offline}, we use implicit maximization to constrain the learning algorithm to learn expectiles using the observed empirical samples. More concretely, we use expectile regression:
\begin{equation}
    \label{eq:practical_m_update}
    \min_\psi \mathcal{L}(\psi) \coloneqq \E{(s,a, g) \sim \rho}{L^\tau_2(M_\psi(s,g)-S_\phi(s,a,g))},
\end{equation}
where $L^\tau_2(u)=|\tau-1(u<0)|u^2$. Intuitively, this step implements the maximization w.r.t $\pi$ by using expectile regression. With the above practical considerations, our objective for learning $S_\phi$ reduces to:

% \begin{multline}
% \scalebox{0.99}{$
% \label{eq:practical_s_update}
% \min_\phi \mathcal{L}(\phi) \coloneqq \beta (1-\gamma)\E{(s, g) \sim \mathcal{D},a \sim \pi_g(\cdot \mid a, g)}{S_\phi(s,\pi_g(s),g)} +\beta \gamma \E{(s,a,g) \sim q,s' \sim p(\cdot|s,a)} { S_\phi(s',\pi_g(s'),g) }\\
%         -\beta \E{(s,a,g) \sim q}{S_\phi(s,a,g)}+\E{(s,a,g) \sim \Mix_\beta(q, \rho)}{(\gamma M_\psi(s',g )-S_\phi(s,a,g))^2},
% $}
% \end{multline}
\begin{equation}
\label{eq:practical_s_update}
\scalebox{0.92}{$
\begin{aligned}
&\min_\phi \mathcal{L}(\phi) \coloneqq \beta (1-\gamma)\E{(s, g) \sim \mathcal{D},a \sim \pi_g(\cdot \mid a, g)}{S_\phi(s,\pi_g(s),g)}  +\beta \gamma \E{(s,a,g) \sim q,s' \sim p(\cdot|s,a)} { S_\phi(s',\pi_g(s'),g) } \\
&\quad -\beta \E{(s,a,g) \sim q}{S_\phi(s,a,g)}+\E{(s,a,g) \sim \Mix_\beta(q, \rho)}{(\gamma M_\psi(s',g )-S_\phi(s,a,g))^2},
\end{aligned}
$}
\end{equation}

% \begin{multline}
% \label{eq:practical_s_update}
% \min_\phi \mathcal{L}(\phi) \coloneqq \beta (1-\gamma)\E{(s, g) \sim \mathcal{D},a \sim \pi_g(\cdot \mid a, g)}{S_\phi(s,\pi_g(s),g)} +\beta \gamma \E{(s,a,g) \sim q,s' \sim p(\cdot|s,a)} { S_\phi(s',\pi_g(s'),g) }\\
%         -\beta \E{(s,a,g) \sim q}{S_\phi(s,a,g)}+\E{(s,a,g) \sim \Mix_\beta(q, \rho)}{(\gamma M_\psi(s',g )-S_\phi(s,a,g))^2},
% \end{multline}

where we have set the offline data distribution as our initial state distribution.  Finally, the policy is extracted via advantage-weighted regression that learns in-distribution actions maximizing the score $S(s,a,g)$:
\begin{equation}
\label{eq:practical_pi_update}
    \min_\theta \mathcal{L}(\theta) \coloneqq \E{(s,a, g) \sim \rho}{\exp(\alpha (S_\phi(s,a,g)-M_\psi(s,g))) \log(\pi_\theta(a|s,g))},
\end{equation}
where $\alpha$ is the temperature parameter. Algorithm~\ref{alg:ALG1} details the practical implementation.

\section{Experiments}
\label{sec:result}

Our experiments study the effectiveness of proposed GCRL algorithm \scogo on a set of simulated benchmarks against other GCRL methods that employ behavior cloning, RL with sparse reward, and contrastive learning. We also analyze if \scogo is robust to environment stochasticity --- a number of prior methods are based on an assumption of deterministic dynamics.  Then, we study if the discriminator-free nature of \scogo is indeed able to prevent performance degradation in the face of low expert coverage in offline data. Finally, we analyze if \texttt{SMORe}'s score-modeling approach helps \scogo scale to a vision-based manipulation offline GCRL benchmark, as density modeling and discriminator learning become increasingly difficult with high-dimensional observations. Hyperparameter ablations can be found in Appendix~\ref{ap:additional_experiments}.

\subsection{Experimental Setup}

Our experiments will use a suite of simulated goal-conditioned tasks extending the tasks from previous work~\citep{ma2022far,plappert2018multi}. In particular we consider the following environments: \texttt{Reacher}, Robotic arm environments - [\texttt{SawyerReach}, \texttt{SawyerDoor}, \texttt{FetchReach}, \texttt{FetchPick}, \texttt{FetchPush}, \texttt{FetchSlide}], Anthropomorphic hand environment - \texttt{HandReach} and Locomotion environments -[\texttt{CheetahTgtVel-me,CheetahTgtVel-re,AntTgtVel-me,AntTgtVel-re}]. Tasks in all environments are specified by a sparse reward function. Depending on whether the task involves object manipulation, the goal distribution is defined over valid configurations in robot or object space. The offline dataset for manipulation tasks consists of transitions collected by a random policy or mixture of 90\% random policy and 10\% expert policy. For locomotion tasks, we generate our dataset using the D4RL benchmark~\citep{fu2020d4rl}, combining a random or medium dataset with 30 episodes of expert data. Note that the policies used to collect the expert locomotion datasets have a different objective than the tasks here, which are to achieve and maintain a particular desired velocity.

% The full discounted return results are shown in Table~\ref{table:offline-gcrl-full-discounted-return}; $(\star)$ indicates statistically significant improvement over the second best method under a two-sample $t$-test. 

\subsection{Offline Goal-conditioned RL benchmark}
\para{Baselines.} 
We compare to state-of-art offline GCRL algorithms, consisting of both regression-based and actor-critic methods. The occupancy-matching based methods are: 
(1) \textbf{GoFar}~\citep{ma2022far}, which derives a dual objective for GCRL based on a coverage assumption. The behavior cloning based methods are:
(1) \textbf{GCSL}~\citep{ghosh2019learning}, which incorporates hindsight relabeling in conjunction with behavior cloning to clone actions that lead to a specified goal, and (2) \textbf{WGCSL}~\citep{yang2022rethinking}, which improves upon GCSL by incorporating discount factor and advantage weighting into the supervised policy learning update. \textbf{Contrastive RL}~\citep{eysenbach2022contrastive} generalizes C-learning~\citep{eysenbach2020c} and represents contrastive GCRL approaches. The RL with sparse reward methods are (1) \textbf{IQL}~\citep{kostrikov2021offline} where we use a state-of-the-art offline RL method repurposed for GCRL along with HER~\citep{andrychowicz2017hindsight} goal sampling, and (2) \textbf{ActionableModel (AM)}~\citep{chebotar2021actionable}, which incorporates conservative Q-Learning~\citep{kumar2020conservative} as well as goal-chaining on top of an actor-critic method. 

The results for all baselines are tuned individually, particularly the best HER ratio was searched among $\{0.2, 0.5, 0.8, 1.0\}$ for each task. \scogo shares the same network architecture for baselines and uses a mixture ratio of $\beta=0.5$. Each method is trained for 10 seeds. Complete architecture and hyperparameter table as well as additional training details are provided in Appendix~\ref{ap:experimental_details}.

\begin{wrapfigure}{r}{0.45\textwidth}
\vspace{-3pt}
    \hspace*{-.75\columnsep}    
\begin{center}
          \includegraphics[width=1.00\linewidth]{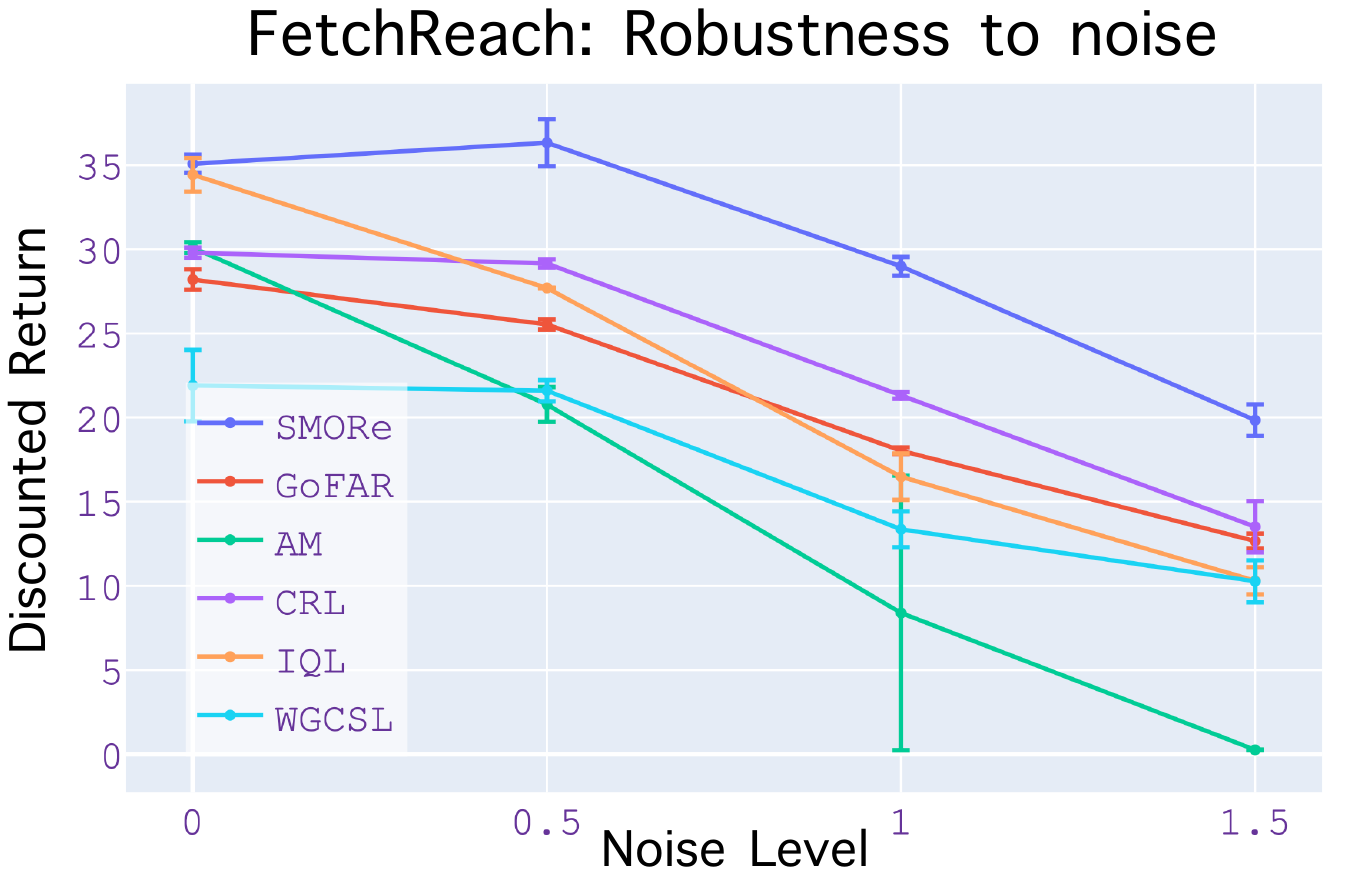}
\end{center}
\vspace{-3.0mm}
\caption{\reb{\scogo is robust in stochastic environments. With increasing noise, \scogo still outperforms prior methods.}}
\vspace{-5pt}
\label{fig:offline_gcrl_with_noise}
\end{wrapfigure}

Table~\ref{table:offline-gcrl-full-discounted-return} reports the \textbf{discounted return} obtained by the learned policy with a sparse binary task reward. ($\star$) denotes statistically significant improvement over the second best method under a \reb{Mann-Whitney U test with a significance level of 0.05}. This metric allows us to compare the algorithms on a finer scale to understand which methods reach the goal as fast as possible and stay in the goal region thereafter for the longest time. Additional results on metrics like success rate and final distance to goal can be found in the appendix. These additional metrics do not take into consideration how \textit{precisely} and \textit{consistently} a goal is being reached. In Table~\ref{table:offline-gcrl-full-discounted-return}, we see that \scogo enjoys a high-performance gain consistently across all tasks in the extended offline GCRL benchmark.

\paragraph{Robustness to environment stochasticity: } We consider a noisy version of the FetchReach environment in this experiment. Gaussian zero-mean noise is added \reb{before executing an action} to generate different variants of the environment with standard deviations of $\set{0.5,1.0,1.5}$. Datasets for these environments are obtained from prior work~\citep{ma2022far}. As we see in Figure~\ref{fig:offline_gcrl_with_noise}, \scogo is robust to stochasticity in the environment, outperforming baselines in terms of discounted return. Behavior cloning based approaches assume deterministic dynamics and are therefore over-optimistic in stochastic environments.

\begin{table}
\vspace{-10pt}
\resizebox{\textwidth}{!}{
\begin{tabular}{l|rr|rr|r|rr}
\toprule
\multicolumn{1}{c|}{\textbf{Task}}
&\multicolumn{2}{c|}{\textbf{Occupancy Matching}}& \multicolumn{2}{c|}{\textbf{Behavior cloning}}& \textbf{Contrastive RL}
& \multicolumn{2}{c}{\textbf{RL+sparse reward}} \\ 
 & \textbf{SMORe} & \textbf{GoFAR}  & \textbf{WGCSL} & \textbf{GCSL} &\textbf{CRL} & \textbf{AM} & \textbf{IQL} \\ 
\midrule 
Reacher ($\star$) & \textbf{28.40}$\pm${\scriptsize 0.88}& 19.74$\pm${\scriptsize 1.35} & 17.57$\pm${\scriptsize 0.53}&  15.87$\pm${\scriptsize 1.31}& 16.44$\pm${\scriptsize0.60 }&23.26 $\pm${\scriptsize 0.14} & 11.70 $\pm${\scriptsize 1.97}\\ 
SawyerReach ($\star$) &\textbf{37.67}$\pm${\scriptsize 0.12} & 15.34$\pm${\scriptsize 0.64} & 15.15$\pm${\scriptsize 0.44} &14.25$\pm${\scriptsize 0.7} &22.32 $\pm${\scriptsize 0.34} & 23.34$\pm${\scriptsize 0.17} &   35.18 $\pm${\scriptsize 0.29}\\ 
SawyerDoor ($\star$) &\textbf{31.48}$\pm${\scriptsize 0.46}& 18.94$\pm${\scriptsize 0.01} & 20.01$\pm${\scriptsize 1.55} & 20.88$\pm${\scriptsize 0.22}& 12.96$\pm${\scriptsize 5.19} &22.12 $\pm${\scriptsize 0.13} &  25.52 $\pm${\scriptsize 1.45}\\
FetchReach ($\star$) & \textbf{35.08}$\pm$ {\scriptsize 0.54} & 28.2 $\pm$ {\scriptsize0.61} & 21.9$\pm$ {\scriptsize2.13}  & 20.91 $\pm$ {\scriptsize2.78} & 30.07$\pm${\scriptsize 0.07}  & 30.1 $\pm$ {\scriptsize0.32}  & 34.43 $\pm$ {\scriptsize1.00}  \\ 
FetchPick ($\star$)&\textbf{26.47} $\pm$ {\scriptsize1.34}& 19.7 $\pm$ {\scriptsize2.57} & 9.84 $\pm$ {\scriptsize2.58} &  7.58$\pm${\scriptsize 1.85}&  0.42$\pm${\scriptsize 0.29} & 8.94 $\pm$ {\scriptsize3.09}  & 16.8 $\pm$ {\scriptsize3.10}  \\
FetchPush ($\star$) &\textbf{26.83}$\pm$ {\scriptsize1.21}& 18.2 $\pm$ {\scriptsize3.00} &14.7 $\pm$ {\scriptsize2.65}  & 13.4 $\pm$ {\scriptsize3.02} &2.40 $\pm${\scriptsize1.28 }  & 14.0 $\pm$ {\scriptsize2.81}  &  22.40 $\pm$ {\scriptsize0.74} \\
FetchSlide &\textbf{4.99}$\pm$ {\scriptsize0.40}& 2.47 $\pm$ {\scriptsize1.44} & 2.73 $\pm$ {\scriptsize1.64}  & 1.75 $\pm$ {\scriptsize1.3} & 0.0$\pm${\scriptsize0.0 }& 1.46 $\pm$ {\scriptsize1.38}  & \textbf{4.80} $\pm$ {\scriptsize1.59}  \\ 
HandReach ($\star$)&\textbf{18.68} $\pm$ {\scriptsize3.35}& 11.5 $\pm$ {\scriptsize5.26}&5.97 $\pm$ {\scriptsize4.81} &1.37 $\pm$ {\scriptsize2.21} & 0.0$\pm${\scriptsize 0.0}& 0.0 $\pm$ {\scriptsize0.0}  & 1.44  $\pm$ {\scriptsize1.77}\\ 
\midrule
CheetahTgtVel-m-e ($\star$)&\textbf{136.71} $\pm$ {\scriptsize10.59} &0.0$\pm$ {\scriptsize0.0} &0.0$\pm$ {\scriptsize0.0} &95.98$\pm$ {\scriptsize15.72} & 0.0$\pm${\scriptsize 0.0}& 0.0$\pm$ {\scriptsize0.0} & 100.38$\pm$ {\scriptsize1.22} \\ 
CheetahTgtVel-r-e ($\star$)&\textbf{60.01} $\pm$ {\scriptsize39.40}& 0.0$\pm$ {\scriptsize0.0} & 0.0$\pm$ {\scriptsize0.0}&11.56 $\pm$ {\scriptsize13.47}& 0.0$\pm${\scriptsize 0.0} & 0.0$\pm$ {\scriptsize0.0} &    0.0$\pm$ {\scriptsize0.0} \\ 
AntTgtVel-m-e &154.95$\pm$ {\scriptsize19.44}& \textbf{168.27}$\pm$ {\scriptsize9.58} &0.0$\pm$ {\scriptsize0.0} &164.54$\pm$ {\scriptsize7.69}& 0.0$\pm${\scriptsize 0.0} & 0.0$\pm$ {\scriptsize0.0} & 148.17 $\pm$ {\scriptsize5.43} \\ 
AntTgtVel-r-e ($\star$)&\textbf{126.22}$\pm$ {\scriptsize14.40}& 74.36$\pm$ {\scriptsize15.97} &0.0$\pm$ {\scriptsize0.0} &104.95$\pm$ {\scriptsize6.00} & 0.0$\pm${\scriptsize 0.0}& 0.0$\pm$ {\scriptsize0.0} &  3.06 $\pm$ {\scriptsize 2.64}  \\ 
\midrule
\bottomrule
\end{tabular}
}
\caption{Discounted Return for the offline GCRL benchmark. Results are averaged over 10 seeds.`m-e' and `r-e' stands for medium-expert mixture and random-expert mixture respectively.  ($\star$) denotes statistically significant improvements.}
\label{table:offline-gcrl-full-discounted-return}
\vspace{-0.2cm}
\end{table}

\subsection{Robustness of Occupancy-Matching Methods to Decreasing Expert Coverage}
\label{sec:scogo_robustness}

We posit that the discriminator-free nature of \scogo makes it more robust to decreasing goal coverage, as it does not suffer from cascading errors stemming from a learned discriminator. In this section, we set out to test this hypothesis by decreasing the amount of expert data in the offline goal-reaching dataset. We compare with GoFAR in Table~\ref{table:offline-gcrl-coverage} due to the similarity between methods and GoFAR's restrictive assumption on coverage of expert data in the suboptimal dataset. Comparison against all the baselines can be found in Appendix~\ref{ap:additional_experiments}. 

Our hypothesis holds true as we see in Table~\ref{table:offline-gcrl-coverage}, the performance of the discriminator-based method GoFar rapidly decays as expert data is decreased in the offline dataset -- 28.4\% with 2.5\%  and 36.15\% with 1\% expert data(i.e. optimal policy's coverage) respectively. \scogo shows a much slower decay in performance, 7.6\% with 2.5\%  and 16\% with 1\% expert data, attesting to the method's robustness under decreasing expert coverage in the offline dataset.

\begin{table}[t]
\vspace{-10pt}
\centering
\resizebox{0.8\textwidth}{!}{
\begin{tabular}{l|cc|cc|cc}
\toprule
\multicolumn{1}{c|}{\textbf{Task}}
& \multicolumn{2}{c|}{\textbf{5 \% expert data}}
& \multicolumn{2}{c|}{\textbf{2.5 \% expert data}} & \multicolumn{2}{c}{\textbf{1 \% expert data}} \\ 
 & \textbf{SMORe} & \textbf{GoFAR}  & \textbf{SMORe} & \textbf{GoFAR}  & \textbf{SMORe} & \textbf{GoFAR}  \\ 
\midrule  
Reacher  &22.43$\pm${\scriptsize3.46} &16.86 $\pm${\scriptsize1.26} &17.92 $\pm$ {\scriptsize0.93}& 12.20$\pm${\scriptsize0.81} &19.61$\pm$ {\scriptsize1.56}  & 11.52 $\pm$ {\scriptsize0.52}\\ 
SawyerReach  & 36.35$\pm${\scriptsize0.37} &13.20 $\pm${\scriptsize1.36} &36.74$\pm${\scriptsize0.62} & 11.57 $\pm${\scriptsize1.79} & 35.44 {\scriptsize0.27}& 9.34$\pm$ {\scriptsize0.17}\\ 
SawyerDoor & 32.82$\pm${\scriptsize0.88}& 20.07$\pm${\scriptsize0.01}&25.69$\pm${\scriptsize0.21} & 19.54$\pm${\scriptsize1.32} & 23.78$\pm${\scriptsize2.88}  & 18.04 $\pm${\scriptsize1.80}\\ 
FetchReach  & 36.00$\pm$ {\scriptsize0.01}& 27.66 $\pm$ {\scriptsize0.55} & 35.58 $\pm$ {\scriptsize0.47}& 27.84 $\pm$ {\scriptsize0.82}  & 35.97 $\pm$ {\scriptsize0.25} & 28.01 $\pm$ {\scriptsize0.20} \\ 
FetchPick &26.43$\pm$ {\scriptsize1.95} & 16.21 $\pm$ {\scriptsize1.46} & 26.17$\pm$ {\scriptsize3.37}& 3.21 $\pm$ {\scriptsize2.22}  & 15.38 $\pm$ {\scriptsize1.52} &  0.31 $\pm$ {\scriptsize 0.31}  \\
FetchPush &23.81$\pm$ {\scriptsize 0.37} & {18.2} $\pm$ {\scriptsize3.00} &22.75$\pm${\scriptsize1.08} & 5.17 $\pm$ {\scriptsize2.01} & 19.04$\pm$ {\scriptsize2.79} & 4.23$\pm$ {\scriptsize3.96}\\
FetchSlide &4.05$\pm$ {\scriptsize1.12} & 1.08 $\pm$ {\scriptsize0.06} & 3.11 $\pm$ {\scriptsize1.61} & 0.96 $\pm$ {\scriptsize0.73} & 3.50$\pm$ {\scriptsize0.97} & 0.86  $\pm$ {\scriptsize1.22}\\ 
\midrule
Average Performance & 25.98& 16.18  & 23.99 & 11.49 &21.81  &  10.33\\
Avg. Perf. Drop & 0 & 0  & -7.6\% & -28.4\% & -16\%  &  -36.15\%\\
\bottomrule
\end{tabular}
}
\vspace{-3pt}
\caption{Discounted Return for the offline GCRL benchmark with 5\%, 2.5\% and 1\% expert data in offline dataset. Results are averaged over 10 seeds.}
\label{table:offline-gcrl-coverage}
\vspace{-10pt}
\end{table}

% \caption{Discounted Return for the offline GCRL benchmark with 5\%, 2.5\% and 1\% expert data in offline dataset. Results are averaged over 10 seeds.}

\subsection{Offline GCRL with image observations}
\begin{figure}[b]
\vspace{-5pt}
\begin{center}
          \includegraphics[width=1.00\linewidth]{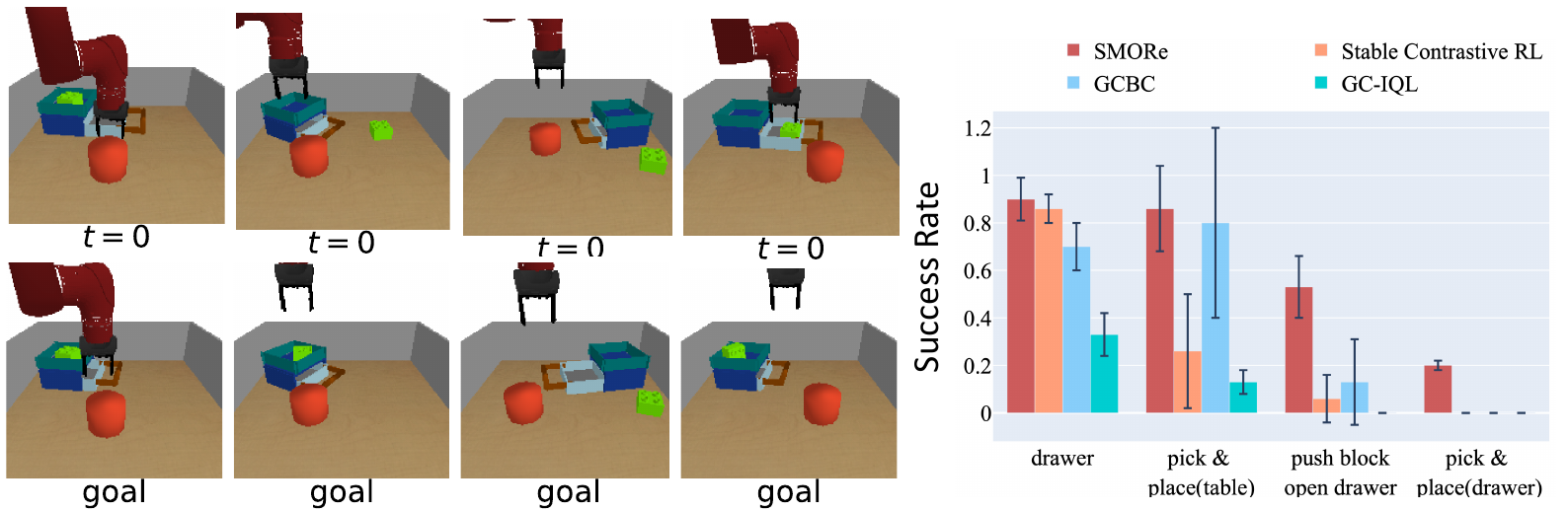}
\end{center}
\vspace{-1.0mm}
\caption{ Evaluation on simulated manipulation tasks with image observations. The left image shows the starting state at the top and the goal at the bottom for evaluation tasks. The error bars show the standard deviation with 5 random seeds. SMORe is competitive or outperforms prior methods on all the tasks we considered.}
\label{fig:offline_gcrl_with_images}
\end{figure}

\scogo provides an effective algorithm for offline GCRL in high-dimensional observation spaces by learning unnormalized scores using a contrastive procedure as opposed to prior works that learn normalized densities~\citep{eysenbach2020c} which are difficult to learn or density ratios~\citep{eysenbach2022contrastive,zheng2023stabilizing} which do not optimize for the optimal goal-conditioned policy in the offline GCRL setting. Similar to prior work~\citep{eysenbach2022contrastive}, we consider the following structure in S-function parameterization to learn performant and generalizable policies: $S(s,a,g) = \phi(s,a)^T\psi(g)$. The $S$-function can be interpreted as the similarity between the two representations given by $\phi$ and $\psi$. Our network architecture for both representations is similar to~\cite{zheng2023stabilizing} and is kept the same across all baselines to ensure a fair comparison of the underlying GCRL method.

We use the offline GCRL benchmark from~\citep{zheng2021stackelberg} which learns goal-reaching policies from an image-observation dataset of 250K transitions with the horizon ranging from 50-100. The benchmark adds another layer of complexity by testing on goals absent from the dataset --- the dataset contains primitive behaviors like picking up objects and pushing drawers but no behavior that completes the compound task we consider from the initial state. The observations and goals are 48x48x3 RGB images. 

\para{Baselines} We compare to the best performing GCRL algorithms from Section~\ref{alg:ALG1} as well as a recent state-of-the-art work, stable contrastive RL~\cite{zheng2023stabilizing}. Stable contrastive RL features a number of improvements over contrastive RL by changing design decisions in neural network architecture, layer normalization, and data augmentation. Since our objective is to compare the quality of the underlying GCRL algorithm, we keep these design decisions consistent across the board.

\para{Results} Figure~\ref{fig:offline_gcrl_with_images} shows the success rate on a variety of unseen tasks for all the methods. \scogo achieves highest success rates across all the tasks, even for the most challenging task of pick, place and closing the drawer. We note that our results differ from~\cite{zheng2023stabilizing} for the baselines as we apply the same design decisions for all methods whereas~\cite{zheng2023stabilizing} focuses on ablating design decisions. 

\section{Related Works}

\textbf{Offline Goal Conditioned Reinforcement Learning.}  Learning to achieve goals in the environment optimally forms the basis of goal-condition RL problems. Studies in cognitive science~\citep{molinaro2023goal} underscore the importance goal-achieving plays in human development. Offline GCRL approaches are typically catered to designing learning algorithms for addressing the sparsity of reward function in the offline setting.  One of the most successful techniques in this setting has been hindsight relabelling. Hindsight-experience relabelling (HER)~\citep{kaelbling1993learning,andrychowicz2017hindsight} suggests relabelling any experience with some commanded goal to the goal that was actually achieved in order to leverage generalization. HER has been investigated in the setting of learning from demonstrations~\citep{ding2019goal} and exploration~\citep{fang2019curriculum} to validate its effectiveness. A number of prior works~\citep{ghosh2019learning,yang2019imitation,chen2020learning,ding2019goal,lynch2020learning,paster2020planning,srivastava2019training,hejna2023distance} have investigated using goal-conditioned behavior cloning, a strategy that uses relabelling to learn goal-conditioned policies, as a way to learn performant policies. \citet{eysenbach2022imitating} shows that this line of work has a limitation of learning suboptimal policies that do not consistently improve over the policy that collected the dataset. The simplest strategy of applying single-task RL to the problem of multi-task goal reaching requires learning a $Q$-function which represents normalized densities over the state-action space. Contrastive RL~\citep{eysenbach2022contrastive,eysenbach2020c,zheng2023stabilizing} emerged as another alternative for GCRL which relabels trajectories and, rather than use that relabelling to learn policies, learns a $Q$-function using a contrastive procedure. While these approaches learn optimal policies in the online setting, they fall behind in the offline setting where they only learn a policy that greedily improves over the $Q$-function of the data collecting policy. Our work learns optimal policies by presenting an off-policy objective that solves GCRL and furthermore learns scores (or unnormalized densities) that alleviate the learning challenges of normalized density estimation.  

\textbf{Distribution matching.} Our approach is inspired by the distribution matching approach~\citep{ghasemipour2020divergence,ni2021f,sikchi2022ranking,swamy2021moments,sikchi2023dual} prominent in imitation learning. \citet{ghasemipour2020divergence,ni2021f} takes the problem of imitating an expert demonstrator in the environment and converts it into a problem of distribution matching between the current policy's state-action visitation distribution and the expert policy's visitation distribution. Indeed, a prior work $f$-PG~\citep{agarwal2024f} proposes a distribution matching approach to GCRL but is restricted to the on-policy setting. Another, prior work~\citep{ma2022far} creates one such distribution matching problem and presents a new optimization problem for GCRL in the form of an off-policy dual~\citep{nachum2020reinforcement,sikchi2023dual}. Such an off-policy dual is very appealing for the offline RL setup, as optimizing for this dual only requires sampling from the offline data distribution. A limitation of their dual construction is the fact that they require learning a discriminator and use that discriminator as the pseudo-reward for solving the GCRL objective. Our approach presents a new construction for GCRL as a distribution matching problem along with a dual construction that leads to a more performant discriminator-free off-policy approach for GCRL.

% Prior work, GoFar~\citep{ma2022far}, explores distribution matching with \textit{action-free} distributions and proposes a dual-objective for offline GCRL. GoFar first learns a discriminator between the offline dataset and the goal distribution and uses that discriminator for downstream RL training. Any errors in the learned discriminator compound and adversely affect the learned policy performance.  
\section{Conclusion}

Prior work in performant online goal-conditioned RL often relies on iterated behavior cloning or contrastive RL. However, these approaches are suboptimal for the offline setting. Existing methods specifically derived for offline GCRL require learning a discriminator and using it as a pseudo-reward, enabling compounding errors that make the resulting policy ineffective. We present an occupancy-matching approach to offline GCRL that provably optimizes a lower bound to the regularized GCRL objective. Our method is discriminator-free, applicable to a number of $f$-divergences, and learns unnormalized scores over actions at a state to reach the goal. We show that these positive aspects of our algorithm allow us to empirically outperform prior methods, stay robust under decreasing goal coverage, and scale to high-dimensional observation space for GCRL.

% \newpage
\subsection*{Acknowledgements}
We thank Siddhant Agarwal, MIDI lab members, and ICLR reviewers for valuable feedback on this work. This work has taken place in the Safe, Correct, and Aligned Learning and Robotics Lab (SCALAR) at The University of Massachusetts Amherst and Machine Intelligence through Decision-making and Interaction (MIDI) Lab at The University of Texas at Austin. SCALAR research is supported in part by the NSF (IIS-2323384), AFOSR (FA9550-20-1-0077), and ARO (78372-CS, W911NF-19-2-0333), and the Center for AI Safety (CAIS).  This research was also sponsored by the Army Research Office under Cooperative Agreement Number W911NF-19-2-0333. HS and AZ are funded in part by a sponsored research agreement with Cisco Systems Inc. The views and conclusions contained in this document are those of the authors and should not be interpreted as representing the official policies, either expressed or implied, of the Army Research Office or the U.S. Government. The U.S. Government is authorized to reproduce and distribute reprints for Government purposes notwithstanding any copyright notation herein.

\bibliography{iclr2024_conference}

\begin{thebibliography}{45}
\providecommand{\natexlab}[1]{#1}
\providecommand{\url}[1]{\texttt{#1}}
\expandafter\ifx\csname urlstyle\endcsname\relax
  \providecommand{\doi}[1]{doi: #1}\else
  \providecommand{\doi}{doi: \begingroup \urlstyle{rm}\Url}\fi

\bibitem[Agarwal et~al.(2019)Agarwal, Jiang, Kakade, and Sun]{agarwal2019reinforcement}
A.~Agarwal, N.~Jiang, S.~M. Kakade, and W.~Sun.
\newblock Reinforcement learning: Theory and algorithms.
\newblock \emph{CS Dept., UW Seattle, Seattle, WA, USA, Tech. Rep}, pages 10--4, 2019.

\bibitem[Agarwal et~al.(2024)Agarwal, Durugkar, Stone, and Zhang]{agarwal2024f}
S.~Agarwal, I.~Durugkar, P.~Stone, and A.~Zhang.
\newblock f-policy gradients: A general framework for goal-conditioned rl using f-divergences.
\newblock \emph{Advances in Neural Information Processing Systems}, 36, 2024.

\bibitem[Al-Hafez et~al.(2023)Al-Hafez, Tateo, Arenz, Zhao, and Peters]{al2023ls}
F.~Al-Hafez, D.~Tateo, O.~Arenz, G.~Zhao, and J.~Peters.
\newblock Ls-iq: Implicit reward regularization for inverse reinforcement learning.
\newblock \emph{arXiv preprint arXiv:2303.00599}, 2023.

\bibitem[Andrychowicz et~al.(2017)Andrychowicz, Wolski, Ray, Schneider, Fong, Welinder, McGrew, Tobin, Pieter~Abbeel, and Zaremba]{andrychowicz2017hindsight}
M.~Andrychowicz, F.~Wolski, A.~Ray, J.~Schneider, R.~Fong, P.~Welinder, B.~McGrew, J.~Tobin, O.~Pieter~Abbeel, and W.~Zaremba.
\newblock Hindsight experience replay.
\newblock \emph{Advances in neural information processing systems}, 30, 2017.

\bibitem[Chebotar et~al.(2021)Chebotar, Hausman, Lu, Xiao, Kalashnikov, Varley, Irpan, Eysenbach, Julian, Finn, et~al.]{chebotar2021actionable}
Y.~Chebotar, K.~Hausman, Y.~Lu, T.~Xiao, D.~Kalashnikov, J.~Varley, A.~Irpan, B.~Eysenbach, R.~Julian, C.~Finn, et~al.
\newblock Actionable models: Unsupervised offline reinforcement learning of robotic skills.
\newblock \emph{arXiv preprint arXiv:2104.07749}, 2021.

\bibitem[Chen et~al.(2020)Chen, Paleja, and Gombolay]{chen2020learning}
L.~Chen, R.~Paleja, and M.~Gombolay.
\newblock Learning from suboptimal demonstration via self-supervised reward regression.
\newblock \emph{arXiv preprint arXiv:2010.11723}, 2020.

\bibitem[Ding et~al.(2019)Ding, Florensa, Abbeel, and Phielipp]{ding2019goal}
Y.~Ding, C.~Florensa, P.~Abbeel, and M.~Phielipp.
\newblock Goal-conditioned imitation learning.
\newblock \emph{Advances in neural information processing systems}, 32, 2019.

\bibitem[Ebert et~al.(2021)Ebert, Yang, Schmeckpeper, Bucher, Georgakis, Daniilidis, Finn, and Levine]{ebert2021bridge}
F.~Ebert, Y.~Yang, K.~Schmeckpeper, B.~Bucher, G.~Georgakis, K.~Daniilidis, C.~Finn, and S.~Levine.
\newblock Bridge data: Boosting generalization of robotic skills with cross-domain datasets.
\newblock \emph{arXiv preprint arXiv:2109.13396}, 2021.

\bibitem[Eysenbach et~al.(2020)Eysenbach, Salakhutdinov, and Levine]{eysenbach2020c}
B.~Eysenbach, R.~Salakhutdinov, and S.~Levine.
\newblock C-learning: Learning to achieve goals via recursive classification.
\newblock \emph{arXiv preprint arXiv:2011.08909}, 2020.

\bibitem[Eysenbach et~al.(2022{\natexlab{a}})Eysenbach, Udatha, Salakhutdinov, and Levine]{eysenbach2022imitating}
B.~Eysenbach, S.~Udatha, R.~R. Salakhutdinov, and S.~Levine.
\newblock Imitating past successes can be very suboptimal.
\newblock \emph{Advances in Neural Information Processing Systems}, 35:\penalty0 6047--6059, 2022{\natexlab{a}}.

\bibitem[Eysenbach et~al.(2022{\natexlab{b}})Eysenbach, Zhang, Levine, and Salakhutdinov]{eysenbach2022contrastive}
B.~Eysenbach, T.~Zhang, S.~Levine, and R.~R. Salakhutdinov.
\newblock Contrastive learning as goal-conditioned reinforcement learning.
\newblock \emph{Advances in Neural Information Processing Systems}, 35:\penalty0 35603--35620, 2022{\natexlab{b}}.

\bibitem[Fang et~al.(2019)Fang, Zhou, Du, Han, and Zhang]{fang2019curriculum}
M.~Fang, T.~Zhou, Y.~Du, L.~Han, and Z.~Zhang.
\newblock Curriculum-guided hindsight experience replay.
\newblock \emph{Advances in neural information processing systems}, 32, 2019.

\bibitem[Fu et~al.(2020)Fu, Kumar, Nachum, Tucker, and Levine]{fu2020d4rl}
J.~Fu, A.~Kumar, O.~Nachum, G.~Tucker, and S.~Levine.
\newblock D4rl: Datasets for deep data-driven reinforcement learning.
\newblock \emph{arXiv preprint arXiv:2004.07219}, 2020.

\bibitem[Garg et~al.(2021)Garg, Chakraborty, Cundy, Song, and Ermon]{garg2021iq}
D.~Garg, S.~Chakraborty, C.~Cundy, J.~Song, and S.~Ermon.
\newblock Iq-learn: Inverse soft-q learning for imitation.
\newblock \emph{Advances in Neural Information Processing Systems}, 34:\penalty0 4028--4039, 2021.

\bibitem[Garg et~al.(2023)Garg, Hejna, Geist, and Ermon]{garg2023extreme}
D.~Garg, J.~Hejna, M.~Geist, and S.~Ermon.
\newblock Extreme q-learning: Maxent rl without entropy.
\newblock \emph{arXiv preprint arXiv:2301.02328}, 2023.

\bibitem[Ghasemipour et~al.(2020)Ghasemipour, Zemel, and Gu]{ghasemipour2020divergence}
S.~K.~S. Ghasemipour, R.~Zemel, and S.~Gu.
\newblock A divergence minimization perspective on imitation learning methods.
\newblock In \emph{Conference on Robot Learning}, pages 1259--1277. PMLR, 2020.

\bibitem[Ghosh et~al.(2019)Ghosh, Gupta, Reddy, Fu, Devin, Eysenbach, and Levine]{ghosh2019learning}
D.~Ghosh, A.~Gupta, A.~Reddy, J.~Fu, C.~Devin, B.~Eysenbach, and S.~Levine.
\newblock Learning to reach goals via iterated supervised learning.
\newblock \emph{arXiv preprint arXiv:1912.06088}, 2019.

\bibitem[Hejna et~al.(2023)Hejna, Gao, and Sadigh]{hejna2023distance}
J.~Hejna, J.~Gao, and D.~Sadigh.
\newblock Distance weighted supervised learning for offline interaction data.
\newblock \emph{arXiv preprint arXiv:2304.13774}, 2023.

\bibitem[Kaelbling(1993)]{kaelbling1993learning}
L.~P. Kaelbling.
\newblock Learning to achieve goals.
\newblock In \emph{IJCAI}, volume~2, pages 1094--8. Citeseer, 1993.

\bibitem[Kostrikov et~al.(2019)Kostrikov, Nachum, and Tompson]{kostrikov2019imitation}
I.~Kostrikov, O.~Nachum, and J.~Tompson.
\newblock Imitation learning via off-policy distribution matching.
\newblock \emph{arXiv preprint arXiv:1912.05032}, 2019.

\bibitem[Kostrikov et~al.(2021)Kostrikov, Nair, and Levine]{kostrikov2021offline}
I.~Kostrikov, A.~Nair, and S.~Levine.
\newblock Offline reinforcement learning with implicit q-learning.
\newblock \emph{arXiv preprint arXiv:2110.06169}, 2021.

\bibitem[Kumar et~al.(2020)Kumar, Zhou, Tucker, and Levine]{kumar2020conservative}
A.~Kumar, A.~Zhou, G.~Tucker, and S.~Levine.
\newblock Conservative q-learning for offline reinforcement learning.
\newblock \emph{Advances in Neural Information Processing Systems}, 33:\penalty0 1179--1191, 2020.

\bibitem[Lynch et~al.(2020)Lynch, Khansari, Xiao, Kumar, Tompson, Levine, and Sermanet]{lynch2020learning}
C.~Lynch, M.~Khansari, T.~Xiao, V.~Kumar, J.~Tompson, S.~Levine, and P.~Sermanet.
\newblock Learning latent plans from play.
\newblock In \emph{Conference on robot learning}, pages 1113--1132. PMLR, 2020.

\bibitem[Ma et~al.(2022)Ma, Yan, Jayaraman, and Bastani]{ma2022far}
Y.~J. Ma, J.~Yan, D.~Jayaraman, and O.~Bastani.
\newblock How far i'll go: Offline goal-conditioned reinforcement learning via $ f $-advantage regression.
\newblock \emph{arXiv preprint arXiv:2206.03023}, 2022.

\bibitem[Manne(1960)]{manne1960linear}
A.~S. Manne.
\newblock Linear programming and sequential decisions.
\newblock \emph{Management Science}, 6\penalty0 (3):\penalty0 259--267, 1960.

\bibitem[Molinaro and Collins(2023)]{molinaro2023goal}
G.~Molinaro and A.~G. Collins.
\newblock A goal-centric outlook on learning.
\newblock \emph{Trends in Cognitive Sciences}, 2023.

\bibitem[Nachum and Dai(2020)]{nachum2020reinforcement}
O.~Nachum and B.~Dai.
\newblock Reinforcement learning via fenchel-rockafellar duality.
\newblock \emph{arXiv preprint arXiv:2001.01866}, 2020.

\bibitem[Ni et~al.(2021)Ni, Sikchi, Wang, Gupta, Lee, and Eysenbach]{ni2021f}
T.~Ni, H.~Sikchi, Y.~Wang, T.~Gupta, L.~Lee, and B.~Eysenbach.
\newblock f-irl: Inverse reinforcement learning via state marginal matching.
\newblock In \emph{Conference on Robot Learning}, pages 529--551. PMLR, 2021.

\bibitem[Oord et~al.(2018)Oord, Li, and Vinyals]{oord2018representation}
A.~v.~d. Oord, Y.~Li, and O.~Vinyals.
\newblock Representation learning with contrastive predictive coding.
\newblock \emph{arXiv preprint arXiv:1807.03748}, 2018.

\bibitem[Padalkar et~al.(2023)Padalkar, Pooley, Jain, Bewley, Herzog, Irpan, Khazatsky, Rai, Singh, Brohan, et~al.]{padalkar2023open}
A.~Padalkar, A.~Pooley, A.~Jain, A.~Bewley, A.~Herzog, A.~Irpan, A.~Khazatsky, A.~Rai, A.~Singh, A.~Brohan, et~al.
\newblock Open x-embodiment: Robotic learning datasets and rt-x models.
\newblock \emph{arXiv preprint arXiv:2310.08864}, 2023.

\bibitem[Paster et~al.(2020)Paster, McIlraith, and Ba]{paster2020planning}
K.~Paster, S.~A. McIlraith, and J.~Ba.
\newblock Planning from pixels using inverse dynamics models.
\newblock \emph{arXiv preprint arXiv:2012.02419}, 2020.

\bibitem[Plappert et~al.(2018)Plappert, Andrychowicz, Ray, McGrew, Baker, Powell, Schneider, Tobin, Chociej, Welinder, et~al.]{plappert2018multi}
M.~Plappert, M.~Andrychowicz, A.~Ray, B.~McGrew, B.~Baker, G.~Powell, J.~Schneider, J.~Tobin, M.~Chociej, P.~Welinder, et~al.
\newblock Multi-goal reinforcement learning: Challenging robotics environments and request for research.
\newblock \emph{arXiv preprint arXiv:1802.09464}, 2018.

\bibitem[R{\'e}nyi(1961)]{renyi1961measures}
A.~R{\'e}nyi.
\newblock On measures of entropy and information.
\newblock In \emph{Proceedings of the Fourth Berkeley Symposium on Mathematical Statistics and Probability, Volume 1: Contributions to the Theory of Statistics}, pages 547--561. University of California Press, 1961.

\bibitem[Sikchi et~al.(2022)Sikchi, Saran, Goo, and Niekum]{sikchi2022ranking}
H.~Sikchi, A.~Saran, W.~Goo, and S.~Niekum.
\newblock A ranking game for imitation learning.
\newblock \emph{arXiv preprint arXiv:2202.03481}, 2022.

\bibitem[Sikchi et~al.(2023)Sikchi, Zheng, Zhang, and Niekum]{sikchi2023dual}
H.~Sikchi, Q.~Zheng, A.~Zhang, and S.~Niekum.
\newblock Dual rl: Unification and new methods for reinforcement and imitation learning.
\newblock In \emph{Sixteenth European Workshop on Reinforcement Learning}, 2023.

\bibitem[Silver et~al.(2014)Silver, Lever, Heess, Degris, Wierstra, and Riedmiller]{silver2014deterministic}
D.~Silver, G.~Lever, N.~Heess, T.~Degris, D.~Wierstra, and M.~Riedmiller.
\newblock Deterministic policy gradient algorithms.
\newblock In \emph{International conference on machine learning}, pages 387--395. Pmlr, 2014.

\bibitem[Srivastava et~al.(2019)Srivastava, Shyam, Mutz, Ja{\'s}kowski, and Schmidhuber]{srivastava2019training}
R.~K. Srivastava, P.~Shyam, F.~Mutz, W.~Ja{\'s}kowski, and J.~Schmidhuber.
\newblock Training agents using upside-down reinforcement learning.
\newblock \emph{arXiv preprint arXiv:1912.02877}, 2019.

\bibitem[Swamy et~al.(2021)Swamy, Choudhury, Bagnell, and Wu]{swamy2021moments}
G.~Swamy, S.~Choudhury, J.~A. Bagnell, and S.~Wu.
\newblock Of moments and matching: A game-theoretic framework for closing the imitation gap.
\newblock In \emph{International Conference on Machine Learning}, pages 10022--10032. PMLR, 2021.

\bibitem[Walke et~al.(2023)Walke, Black, Lee, Kim, Du, Zheng, Zhao, Hansen-Estruch, Vuong, He, et~al.]{walke2023bridgedata}
H.~Walke, K.~Black, A.~Lee, M.~J. Kim, M.~Du, C.~Zheng, T.~Zhao, P.~Hansen-Estruch, Q.~Vuong, A.~He, et~al.
\newblock Bridgedata v2: A dataset for robot learning at scale.
\newblock \emph{arXiv preprint arXiv:2308.12952}, 2023.

\bibitem[Wu et~al.(2019)Wu, Tucker, and Nachum]{wu2019behavior}
Y.~Wu, G.~Tucker, and O.~Nachum.
\newblock Behavior regularized offline reinforcement learning.
\newblock \emph{arXiv preprint arXiv:1911.11361}, 2019.

\bibitem[Xu et~al.(2023)Xu, Jiang, Li, Yang, Wang, Chan, and Zhan]{xu2023offline}
H.~Xu, L.~Jiang, J.~Li, Z.~Yang, Z.~Wang, V.~W.~K. Chan, and X.~Zhan.
\newblock Offline rl with no ood actions: In-sample learning via implicit value regularization.
\newblock \emph{arXiv preprint arXiv:2303.15810}, 2023.

\bibitem[Yang et~al.(2019)Yang, Ma, Huang, Sun, Liu, Huang, and Gan]{yang2019imitation}
C.~Yang, X.~Ma, W.~Huang, F.~Sun, H.~Liu, J.~Huang, and C.~Gan.
\newblock Imitation learning from observations by minimizing inverse dynamics disagreement.
\newblock \emph{arXiv preprint arXiv:1910.04417}, 2019.

\bibitem[Yang et~al.(2022)Yang, Lu, Li, Sun, Fang, Du, Li, Han, and Zhang]{yang2022rethinking}
R.~Yang, Y.~Lu, W.~Li, H.~Sun, M.~Fang, Y.~Du, X.~Li, L.~Han, and C.~Zhang.
\newblock Rethinking goal-conditioned supervised learning and its connection to offline rl.
\newblock \emph{arXiv preprint arXiv:2202.04478}, 2022.

\bibitem[Zheng et~al.(2023)Zheng, Eysenbach, Walke, Yin, Fang, Salakhutdinov, and Levine]{zheng2023stabilizing}
C.~Zheng, B.~Eysenbach, H.~Walke, P.~Yin, K.~Fang, R.~Salakhutdinov, and S.~Levine.
\newblock Stabilizing contrastive rl: Techniques for offline goal reaching.
\newblock \emph{arXiv preprint arXiv:2306.03346}, 2023.

\bibitem[Zheng et~al.(2021)Zheng, Fiez, Alumbaugh, Chasnov, and Ratliff]{zheng2021stackelberg}
L.~Zheng, T.~Fiez, Z.~Alumbaugh, B.~Chasnov, and L.~J. Ratliff.
\newblock Stackelberg actor-critic: Game-theoretic reinforcement learning algorithms.
\newblock \emph{arXiv preprint arXiv:2109.12286}, 2021.

\end{thebibliography}
\bibliographystyle{abbrvnat}

\newpage
\appendix
\section{Appendix}

% \subsection{Code Release}
%  The accompanying JAX code and instructions to reproduce the results for this work can be found at the \href{https://drive.google.com/drive/folders/16DG1zWrqpxctldMoP3mZwxlxDF5fQvHI?usp=sharing}{link here}.

\subsection{Theory}

In this section, we first show the equivalence of the GCRL problem and the distribution-matching objective of imitation learning. Then, we show how the mixture distribution objective relates to offline GCRL objective. Finally, we derive the dual objective for mixture distribution matching that leads to our method \scogo. 

\subsubsection{Reduction of GCRL to distribution matching}

\gcrlimitationf*

\begin{proof}
This proof is adapted from~\cite{ma2022far} for goal transition distributions and state-action distributions. Let $Z=\int e^{R(s,a,g)}~ds~da~dg$ and $\alpha>0$ be the temperatue parameter. Note that $q(s,a,g)=e^{r(s,a,g)}$ where r is defined in the proposition, strictly generalizes the original definition $q(s,a,g) = q^{\texttt{train}}(g)\E{s'\sim p(\cdot|s,a)}{\mathbb{I}(\phi(s')=g)}$ and recovers it when $\alpha\to \infty$. Starting with the true GCRL objective:
\begin{align}
    \alpha J(\pi_g) &= \E{d^{\pi_g}}{\alpha R(s,a,g)}\\
    &= \E{d^{\pi_g}}{\log e^{\alpha R(s,a,g)}}\\
    &= \E{d^{\pi_g}}{\log( \frac{e^{\alpha R(s,a,g)}}{Z}\frac{d^{\pi_g}(s,a,g)}{d^{\pi_g}(s,a,g)}Z)}\\
    &= \E{d^{\pi_g}}{\log( \frac{q(s,a,g)}{d^{\pi_g}(s,a,g)}Z)} +\E{d^{\pi_g}}{\log d^{\pi_g}} \\
    &= -D_{KL}(d^{\pi_g}(s,a,g)\|q(s,a,g)) - \mathcal{H}(d^{\pi_g})+\log(Z)
\end{align}
Rearranging terms we get:
\begin{equation}
    J(\pi_g) + \frac{1}{\alpha} \mathcal{H}(d^{\pi_g}) =- \frac{1}{\alpha} D_{KL}(d^{\pi_g}(s,a,g)\| q(s,a,g)) + C
\end{equation}

For any $f$-divergence that upper bounds 
 the KL divergence we have:
 \begin{equation}
    J(\pi_g) + \frac{1}{\alpha} \mathcal{H}(d^{\pi_g}) = - \frac{1}{\alpha} D_{KL}(d^{\pi_g}(s,a,g)\| q(s,a,g)) + C \ge - \frac{1}{\alpha} D_{f}(d^{\pi_g}(s,a,g)\| q(s,a,g)) +C
\end{equation}
\end{proof}

\textbf{A (dataset) regularized GCRL objective: } Define a regularized objective for GCRL as follows:
\begin{equation}    
J_{offline}(\pi) = \alpha_1 \E{d^\pi}{ e^{r(s,a,g)}} + \alpha_2 \E{d^\pi(s,a,g)}{\rho(s,a,g)}.
\end{equation}

\reb{The second term in the objective $\E{d^\pi(s,a,g)}{\rho(s,a,g)}$ above is maximized when the policy visitation places more probability mass on the most visited transitions in the dataset. To see why this is, consider two probability distributions represented as vectors $d^\pi$ and $\rho$ with individuals elements of the vector indexed by $i$:
\begin{equation}
    \langle d^\pi,\rho \rangle \le \max_i \rho_i
\end{equation}
The equality holds only when $d^\pi$ places probability mass on all state-action-goal tuples which are most visited in the offline dataset $\rho$. The first term maximizes the true GCRL objective while the second term prefers staying close to transitions that are most frequently observed in the offline dataset. A constraint of $\langle d^\pi,\rho \rangle \ge 1-\delta$ implies that the agent visitation places atleast half of the probability mass on state-action-goal tuples whose average visitation in offline dataset is greater than or equal to $1-\delta$. With weights $\alpha_1$ and $\alpha_2$, the objective above reflects an lagrangian relaxation to this constraint. Thus the above offline objective presents an alternative offline objective when compared to the classical offline RL objectives~\cite {wu2019behavior,nachum2020reinforcement}.} 
% in constraining the visitation of the learned policy, as the second objective is maximized when $d^\pi(s,a,g)=\rho(s,a,g)$. 
% The second objective can also be interpreted as the unscaled cosine similarity between the two distributions.
% Also, a constraint of $\E{d^\pi(s,a,g)}{\rho(s,a,g)}>1-\delta$ implies that $d^\pi(s,a,g)$ has atleast $1-\delta$ coverage of the offline data distribution.

Proposition~\ref{lemma:gcrl-imitation-mixture}  derives the connection between the dataset regularized GCRL objective and \scogo:
\begin{restatable}[]{proposition}{gcrlimitationfmixture}
\label{lemma:gcrl-imitation-mixture}
Consider a stochastic MDP, a stochastic policy $\pi$, and a sparse reward function $r(s,a,g)=\E{s'\sim p(\cdot|s,a)}{\mathbb{I}(\phi(s')=g,q^{train}(g)>0)}$ where $\mathbb{I}$ is an indicator function, define a soft goal transition distribution to be $q(s,a,g)~\propto ~\text{exp}(\alpha ~r(s,a,g))$ the following bounds hold for any $f$-divergence that upper bounds KL-divergence (eg. $\chi^2$, Jensen-Shannon):
\begin{equation}
      \log J_{offline}(\pi_g) +\mathcal{H}(\dmix) +C \ge - \D_f(\dmix \| \demix),
\end{equation}
where $\mathcal{H}$ denotes the entropy, $\alpha$ is a temperature parameter, $\alpha_1=\beta^2$, $\alpha_2=\beta(1-\beta)Z$ and C is a positive constant. Furthermore, the bound is tight when $f$ is the KL-divergence.
\end{restatable}

\begin{proof}

We first consider the following two objectives for GCRL and show that they are equivalent. This reduction will later help in proving a connection to mixture occupancy matching. We consider $\alpha=1$ w.l.o.g.
Here are two objectives we consider:
\begin{equation}
    J(\pi) = \E{d^\pi}{r(s,a,g)}
\end{equation}
\begin{equation}
    J'(\pi) = \E{d^\pi}{e^{r(s,a,g)}}
\end{equation}

In GCRL reward functions are sparse and binary. We show the equivalence of first two objectives in find the optimal goal conditioned policy via two arguments. First, notice that the rewards for goal transition states for objective $J'(\pi)$ is $e$ and 1 for all other transitions. This is in contrast to $J(\pi)$ which considers a reward function $1$ at goal transitions states and 0 otherwise. Under our assumption of infinite horizon discounted MDP, we can translate the rewards while keeping the optimal policy same in MDP considered by $J'(\pi)$ to $e-1$ at goal transitions states and $0$ otherwise. Further we can scale the rewards by $1/(e-1)$ and recover and MDP with same optimal policy that has reward of $1$ at goal-transition states and 0 otherwise. This concludes the equivalence of maximizing $J'(\pi)$ as an alternative to $J(\pi)$ while recovering the same optimal policy.

We now consider a regularized (pessimistic/offline) GCRL problem with the shifted reward functions $e^{r(s,a,g)}$ that maximizes the reward while ensuring the policy visitation stays close to offline data visitation in cosine similarity.
\begin{equation}    
J_{offline}(\pi) = \alpha_1 \E{d^\pi}{ e^{r(s,a,g)}} + \alpha_2 \E{d^\pi(s,a,g)}{\rho(s,a,g)}.
\end{equation}

With a particular instantiation of hyperparameters we show that the $J_{offline}(\pi)$ objective can be simplified to an equivalent objective $J'_{offline}(\pi)$ by setting $\alpha_1 = \beta^2$ and $\alpha_2 = \beta (1-\beta) Z$ where $Z$ is the partition function for $e^{r(s,a,g)}$ over entire $\mathcal{S}\times \mathcal{A} \times \mathcal{G}$.

\begin{equation}
    J'_{offline}(\pi) =\E{\dmix}{\beta e^{r(s,a,g)}+ (1-\beta)\rho(s,a,g).Z}
\end{equation}

\begin{align}
    J'_{offline}(\pi) &=\E{\dmix}{\beta e^{r(s,a,g)}+ (1-\beta)\rho(s,a,g).Z}\\
    &= \beta^2 \E{d^\pi}{e^{r(s,a,g)}} + \beta (1-\beta)Z\E{d^\pi}{\rho(s,a,g)} \\&+ (1-\beta)\E{d^O}{\beta e^{r(s,a,g)}+ (1-\beta)\rho(s,a,g).Z}\beta \\
\end{align}
\begin{align}
        &= \beta^2 \E{d^\pi}{e^{r(s,a,g)}} + \beta (1-\beta)Z\E{d^\pi}{\rho(s,a,g)} + C'\\
    % &=  \beta^2 \E{d^\pi}{e^{r(s,a,g)}} - \beta (1-\beta)Z \D_{\chi^2}[\rho \| d^\pi] -\beta (1-\beta)Z  \mathcal{H}(d^\pi)+ C'+\beta (1-\beta)Z\\
    \label{eq:comparing_offline_objectives}
    & = J_{offline}(\pi) + C'
\end{align}

Now that we have shown $J'_{offline}(\pi) \equiv J_{offline}(\pi)$ and hence solving the same optimization problem, we proceed to derive connections with mixture occupancy matching which follows through an application of Jensen's inequality:

\begin{align}
    \log J'_{offline}(\pi) &= \log \E{\dmix}{\beta e^{r(s,a,g)}+ (1-\beta)\rho(s,a,g).Z}\\
    &\ge  \E{\dmix}{\log(\beta e^{r(s,a,g)}+ (1-\beta)\rho(s,a,g).Z)}\\
\end{align}
\begin{align}
        &= \E{\dmix}{\log(\beta q(s,a,g)+ (1-\beta)\rho(s,a,g))}+\log Z\\
    &= -D_{KL}[\dmix\|\demix]  - \mathcal{H}(\dmix) +\log Z
\end{align}

For any $f$-divergence that upperbounds 
 the KL divergence since $Z\ge 1$ we have:
 \begin{equation}
   \log J'_{offline}(\pi) + \frac{1}{\alpha} \mathcal{H}(\dmix) \ge - \frac{1}{\alpha} D_{f}(\dmix\| \demix)  
\end{equation}

Further simplifying using Eq~\ref{eq:comparing_offline_objectives}:
 \begin{equation}
   \log J_{offline}(\pi) + \frac{1}{\alpha} \mathcal{H}(\dmix+C \ge - \frac{1}{\alpha} D_{f}(\dmix\| \demix)  
\end{equation}

\end{proof}

Optimizing the mixture distribution matching objective of \scogo maximizes a variant of \textit{offline/dataset regularized} GCRL objective where the entropy for distribution $\dmix$ is jointly maximized. Therefore we have shown that the minimizing discrepancy of mixture distribution occupancy maximizes a lower bounds to an offline variant of maxent GCRL objective. 

\subsection{Convex Conjugates and $f$-divergences}
\label{ap:convex_conjugation}
We first review the basics of duality in reinforcement learning.  Let $f:\mathbb{R}_+\to \mathbb{R}$ be a convex function. The convex conjugate $f^*: \mathbb{R}_+ \rightarrow \mathbb{R}$ of $f$ is defined by:
\begin{equation}
    \label{eq:cx_conjugate_def}
    f^*(y)= \text{sup}_{x \in \mathbb{R}_+}[xy -f(x)].
\end{equation}
The convex conjugates have the important property that $f^*$ is also convex and the convex conjugate of $f^*$ retrieves back the original function $f$. 
We also note an important relation regarding $f$ and $f^*$: $(f^*)^{'}=(f')^{-1}$, where the $'$ notation denotes first derivative.

Going forward, we would be dealing extensively with $f$-divergences. Informally, $f$-divergences~\citep{renyi1961measures} are a measure of distance between two probability distributions. Here's a more formal definition:

Let $P$ and $Q$ be two probability distributions over a space $\mathcal{Z}$ such that $P$ is absolutely continuous with respect to $Q$ \footnote{
Let $z$ denote the random variable. 
For any measurable set $Z \subseteq \mathcal{Z}$, $Q(z \in Z) = 0$ implies $P(z \in Z) = 0$.}. 
For a function $f: \mathbb{R}_+ \to \mathbb{R}$ that is a convex lower semi-continuous and $f(1)=0$,
the $f$-divergence of $P$ from $Q$ is  
\begin{equation}
\f{P}{Q}=\mathbb{E}_{z\sim Q}\left[f\left(\frac{P(z)}{Q(z)}\right)\right].
 \end{equation}
Table~\ref{tbl:div} lists some common $f$-divergences with their generator functions $f$ and the conjugate functions $f^*$.
\begin{table}[t]
	\centering
	\vskip5pt
	\def\arraystretch{1.5}
	\begin{tabular}{l|c|c}
        \toprule
		Divergence Name & Generator $f(x)$ & Conjugate $f^*(y)$  \\ \hline
		 KL (Reverse) & $x\log{x}$ & $e^{(y-1)}$  \\
		Squared Hellinger & $(\sqrt{x} - 1)^2$ & $\frac{y}{1-y}$ \\ 
		Pearson $\chi ^{2}$ & $(x-1)^2$ & $y + \frac{y^2}{4}$ \\
		Total Variation & $\frac{1}{2} | x-1| $ & $y$ if $y \in [-\frac{1}{2}, \frac{1}{2}]$ otherwise $\infty$  \\
		Jensen-Shannon & $-(x+1)\log(\frac{x+1}{2}) + x \log{x}$ & $-\log{(2- e^y)}$ \\
  \bottomrule
	\end{tabular}
 \vskip5pt
	\caption{List of common $f$-divergences.} 
	\label{tbl:div}
\end{table}

\subsection{SMORe: Dual objective for Offline Goal conditioned reinforcement learning}
\label{ap:closer}
In this section, we derive the dual objective for solving the multi-task occupancy problem formulation for GCRL. First, we derive the original variant of \scogo for the GCRL problem and later derive the action-free \scogo variant for the interested readers.
\newpage
\scogoQ*

\begin{proof}

Recall that: $\dmix :=\beta d(s,a,g)+(1-\beta)\rho(s,a,g)$ and $\demix := \beta q(s,a,g)+(1-\beta)\rho(s,a,g)$. $\dmix$ denotes the mixture between the current agent's joint-goal visitation distribution with an offline transition dataset potentially suboptimal and $\demix$ is the mixture between the expert's visitation distribution with arbitrary experience from the offline transition dataset. Minimizing the divergence between these visitation distributions still solves the occupancy problem, i.e $d^{\pi_g}=q$ when $q$ is achievable. We start with the primal formulation from Eq~\ref{eq:primal_gcrl_f_mixture} for mixture divergence regularization:
\begin{align*}
    &\max_{d(s,a,g)\ge0,\pi(a|s)}  -\f{\dmix}{\demix} \nonumber\\
    &\text{s.t}~~d(s,a,g)=(1-\gamma)\rho_0(s,g).\pi(a|s,g)+\gamma \pi(a|s,g)\sum_{s',a'} d(s',a',g)p(s|s',a').
\end{align*}
Applying Lagrangian duality and convex conjugate~\eqref{eq:cx_conjugate_def} to this problem, we can convert it to an unconstrained problem with dual variables $S(s, a,g)$ defined for all $s, a \in \S \times \A \times \mathcal{G}$:
\begin{align}
    &\max_{\pi,d\ge0} \min_{S(s,a,g)} -\f{\dmix}{\demix}\nonumber\\
    &+ \sum_{s,a,g} S(s,a,g)\left((1-\gamma)d_0(s,g).\pi(a|s,g)+\gamma \sum_{s',a'} d(s',a',g) p(s|s',a')\pi(a|s,g)-d(s,a,g)\right)\\
    &= \max_{\pi,d\ge0} \min_{S(s,a,g)}   (1-\gamma)\E{d_0(s,g),\pi(a|s,g)}{S(s,a,g)} \nonumber \\
    &+  \E{s,a,g\sim d}{\gamma \sum_{s',a'} p(s'|s,a)\pi(a'|s')S(s',a',g)-S(s,a,g)}\\
    &-\f{\dmix}{\demix}
\end{align}
\newline
% \todoah{$d$ as an optimization varible is not required to be a valid distribution, so we can't write here expectation over d.} \hs{d is restricted to simplex over $\mathcal{S}\times\mathcal{A} \times \mathcal{G}$. The constraints only say that it has to be valid distribution for some policy}
\begin{align}
&= \max_{\pi,d\ge0} \min_{S(s,a,g)}   \beta(1-\gamma)\E{d_0(s,g),\pi(a|s,g)}{S(s,a,g)} \nonumber \\
    &+ \beta \E{s,a,g\sim d}{\gamma \sum_{s',a'} p(s'|s,a)\pi(a'|s')S(s',a',g)-S(s,a,g)}\nonumber\\
\end{align}
\begin{align}
        &+(1-\beta) \E{s,a,g\sim \rho}{\gamma \sum_{s',a'} p(s'|s,a)\pi(a'|s')S(s',a',g)-S(s,a,g)}\nonumber\\
    &-(1-\beta) \E{s,a,g\sim \rho}{\gamma \sum_{s',a'} p(s'|s,a)\pi(a'|s',g)S(s',a',g)-S(s,a,g)}\\
    &-\f{\dmix}{\demix}
\end{align}

% \todoah{I don't see how the $\beta$ variable comes out to multiply the second term of the equation (25)? Oh ! I see, you can do a variable change $S \rightarrow \beta S$, but then you need multiply the first term by $\beta$ as well} \todoah{and what is the distribution $d^S$. it is not defined earlier. you maybe mean $d^O$?}
Now using the fact that strong duality holds in this problem we can swap the inner max and min resulting in:
\begin{align}
&= \max_{\pi} \min_{S(s,a,g)}  \max_{\dmix \ge0} \beta(1-\gamma)\E{d_0(s,g),\pi(a|s,g)}{S(s,a,g)} \nonumber \\
    &+ \beta \E{s,a,g\sim d}{\gamma \sum_{s',a'} p(s'|s,a)\pi(a'|s')S(s',a',g)-S(s,a,g)}\nonumber\\
    &+(1-\beta) \E{s,a,g\sim \rho}{\gamma \sum_{s',a'} p(s'|s,a)\pi(a'|s')S(s',a',g)-S(s,a,g)}\nonumber\\
    &-(1-\beta) \E{s,a,g\sim \rho}{\gamma \sum_{s',a'} p(s'|s,a)\pi(a'|s',g)S(s',a',g)-S(s,a,g)}\\
    &-\f{\dmix}{\demix}\\
\end{align}

We can now apply the convex conjugate (Eq.~\eqref{eq:cx_conjugate_def}) definition to obtain a closed form for the inner maximization problem simplifying to:
\begin{align}
    &\max_{\pi(a|s,g)}\min_{S(s,a,g)} \beta (1-\gamma)\E{d_0(s,g),\pi(a|s,g)}{S(s,a,g)} \nonumber\\
    &+\E{s,a,g\sim \demix}{f^*(\gamma \sum_{s',a'} p(s'|s,a,g)\pi(a'|s')S(s',a',g)-S(s,a,g))}\nonumber\\
    & - (1-\beta) \E{s,a,g\sim \rho}{\gamma \sum_{s',a'} p(s'|s,a,g)\pi(a'|s')S(s',a',g)-S(s,a,g)}
\end{align}
This completes our derivation of the \scogo objective. Since strong duality holds (objective convex, constraints linear and feasible), \scogo and the primal mixture occupancy matching share the same global optima $\pi^*_g$.
\end{proof}

\subsection{Action-free SMORe: Dual-V objective for offline goal conditioned reinforcement learning}

The primal problem in Equation~\ref{eq:primal_gcrl_f_mixture} is over-constrained. The objective determines the visitation distribution $d$ uniquely under a fixed policy. It turns out we can further relax this constraint to get an objective that results in the same optimal solution~\citep{agarwal2019reinforcement} $\pi^*_g$ by rewriting our primal formulation as:
\begin{align}
\label{eq:primal_gcrl_v}
    &\max_{d(s,a,g)\ge0}  -\f{\dmix}{\demix} \nonumber\\
    &\text{s.t}~~\sum_a d(s,a,g)=(1-\gamma)\rho_0(s,g)+\gamma \sum_{s',a'} d(s',a',g)p(s|s',a').
\end{align}

\begin{restatable}[]{theorem}{scogoV}
    \label{lemma:dual_gcrl_v}
    Let $y(s,a,g)= \gamma \E{s'\sim p(\cdot|s,a)}{S(s',g)}-S(s,g)$. The action-free dual problem to the multi-task mixture occupancy matching objective (Equation ~\ref{eq:primal_gcrl_v}) is given by:
    \begin{multline}
     \min_{S(s,g)}  \beta (1-\gamma)\E{d_0(s,g)}{S(s,g)} \\+\E{s,a,g\sim \demix}{\max\left(0, (f')^{-1} \left(y(s,a,g)\right)\right)y(s,a,g)- f\left(\max\left(0, (f')^{-1} \left(y(s,a,g)\right)\right)\right)}\nonumber\\
    - (1-\beta) \E{s,a,g\sim \rho}{\gamma \sum_{s'} p(s'|s,a)S(s',g)-S(s,g)}
    \end{multline}
where $S$ is the lagrange dual variable defined as $S:\mathcal{S}\times\mathcal{G}\to \mathbb{R}$ . Moreover, strong duality holds from Slater's conditions the primal and dual share the same optimal solution $\pi^*_g$ for any offline transition distribution $d^O$.
\end{restatable}

\begin{proof}

Proceeding as before and applying Lagrangian duality and convex conjugate~\eqref{eq:cx_conjugate_def} to this problem, we can convert it to an unconstrained problem with dual variables $S(s,g)$ defined for all $s, g \in \S \times \mathcal{G}$:

\begin{align}
    &\max_{d\ge0} \min_{S(s,g)} -\f{\dmix}{\demix}\nonumber\\
    &+ \sum_{s,g} S(s,g)\left((1-\gamma)d_0(s,g)+\gamma \sum_{s',a',g} d(s',a',g) p(s|s',a',g)-\sum_a d(s,a,g)\right)\\
    &= \max_{d\ge0} \min_{S(s,g)}   (1-\gamma)\E{d_0(s,g)}{S(s,g)} \nonumber \\
    &+  \E{s,a,g\sim d}{\gamma \sum_{s'} p(s'|s,a)\pi(a'|s')S(s',g)-S(s,g)}\\
    &-\f{\dmix}{\demix}
\end{align}
\newline
\begin{align}
&= \max_{d\ge0} \min_{S(s,g)}  \beta (1-\gamma)\E{d_0(s,g)}{S(s,g)} \nonumber \\
    &+ \beta \E{s,a,g\sim d}{\gamma \sum_{s'} p(s'|s,a)S(s',g)-S(s,g)}\nonumber\\
    &+(1-\beta) \E{s,a,g\sim d^O}{\gamma \sum_{s'} p(s'|s,a)S(s',g)-S(s,g)}\nonumber\\
    &-(1-\beta) \E{s,a,g\sim d^O}{\gamma \sum_{s'} p(s'|s,a)S(s',g)-S(s,g)}\\
    &-\f{\dmix}{\demix}
\end{align}

Now using the fact that strong duality holds in this problem we can swap the inner max and min resulting in:
\begin{align}
\label{eq:scogo_v_mid}
&= \min_{S(s,g)}  \max_{\dmix \ge0} \beta(1-\gamma)\E{d_0(s,g)}{S(s,g)} \nonumber \\
    &+ \beta \E{s,a,g\sim d}{\gamma \sum_{s'} p(s'|s,a)S(s',g)-S(s,g)}\nonumber\\
    &+(1-\beta) \E{s,a,g\sim d^O}{\gamma \sum_{s'} p(s'|s,a)S(s',g)-S(s,g)}\nonumber\\
    &-(1-\beta) \E{s,a,g\sim d^O}{\gamma \sum_{s'} p(s'|s,a)S(s',g)-S(s,g)}\\
    &-\f{\dmix}{\demix}
\end{align}

Unlike previous case where constraints uniquely define a valid $d$ for any given $\pi$, in this case we need to take into account the hidden constraint $d\ge0$ or equivalently $\dmix \ge 0$. To incorporate the non-negativity constraints we consider the inner maximization separately and derive a closed-form solution that adheres to the non-negativity constraints. Recall $y(s,a,g) = \E{s'\sim p(s,a)}{S(s',g)}- S(s,g)$.

\begin{multline}
    \max_{\dmix\ge0} \E{s,a,g\sim \dmix}{\gamma \sum_{s'} p(s'|s,a)S(s',g)-S(s,g)}\nonumber\\-\f{\dmix}{\demix} 
\end{multline}
 We can now construct the Lagrangian dual to incorporate the constraint $\dmix \ge0$ in its equivalent form $w(s,a,g)\ge 0$  and obtain the following where $w\overset{\Delta}{=}\frac{\dmix}{\demix}$:

\begin{align}
    \label{eq:lagrangian_recoil}
    \max_{w(s,a,g)}\max_{\lambda\ge 0 } \E{s,a\sim \demix}{w(s,a,g)y(s,a,g)}-\E{\demix}{f(w(s,a,g))} + \sum_{s,a,g} \lambda (w(s,a,g)-0)
\end{align}

Since strong duality holds, we can use the KKT constraints to find the solutions $w^*(s,a,g)$ and $\lambda^*(s,a,g)$.

\begin{enumerate}
    \item \textbf{Primal feasibility}: $w^*(s,a,g)\ge 0~~\forall~s,a$
    \item \textbf{Dual feasibility}: $\lambda^*\ge0~~\forall~s,a$
    \item \textbf{Stationarity}: $\demix(-f'(w^*(s,a,g))+y(s,a,g)+\lambda^*(s,a,g))=0~~\forall~s,a$
    \item \textbf{Complementary Slackness}: $(w^*(s,a,g)-0)\lambda^*(s,a,g)=0~~\forall~s,a$
\end{enumerate}
% \textbf{1. Primal feasibility}: $w^*(s,a)\ge 0~~\forall~s,a$\\
% \textbf{2. Dual feasibility}: $\lambda^*\ge0~~\forall~s,a$\\
% \textbf{3. Stationarity}: $\demix(f'(w^*(s,a))+y(s,a,g)+\lambda^*(s,a))=0~~\forall~s,a$\\
% \textbf{4. Complementary Slackness}: $(w^*(s,a)-0)\lambda^*(s,a)=0~~\forall~s,a$

Using stationarity we have the following:
\begin{equation}
    f'(w^*(s,a,g)) = y(s,a,g)+\lambda^*(s,a,g)~~\forall~s,a,g
\end{equation}
% Now note that the definition of $f$-divergence only considers a functional mapping $f:(0,\infty)\to \mathbb{R}$. This makes $f'$ and correspondingly ${f'}^{-1}$ ill-defined when the input is less than equal to zero. We define 
Now using complementary slackness,
only two cases are possible $w^*(s,a,g)\ge 0$ or $\lambda^*(s,a,g)\ge 0$.  
% \qq{should write these two cases are exclusive? can have both zeros in complementary slackness?}
Combining both cases we arrive at the following solution for this constrained optimization:
\begin{equation}
    w^*(s,a) = \max\left(0,{f'}^{-1}(y(s,a,g)) \right)
\end{equation}

% We can still find a closed-form solution for the inner optimization, in the case when $d\ge 0$, although a bit more involved (See Appendix for the proof). Let $y(s,a) =r(s,a)+\gamma \sum_{s'} p(s'|s,a)\pi(a'|s')Q(s',a')-Q(s,a) $. Also let $p(s,a) = \frac{(1-\alpha)\rho^R(s,a)}{\alpha \rho^E(s,a) + (1-\alpha)\rho^R(s,a)}$. 

Using the optimal closed-form solution  ($w^*$) for $\dmix$ of the inner optimization in  Eq.~\eqref{eq:scogo_v_mid}  we obtain

\begin{align}
 &\min_{S(s,a)}  \beta (1-\gamma)\E{d_0(s)}{S(s,g)} \nonumber\\
    &+\E{s,a,g\sim \demix}{\max\left(0, (f')^{-1} \left(y(s,a,g)\right)\right)y(s,a,g)-\alpha f\left(\max\left(0, (f')^{-1} \left(y(s,a,g)\right)\right)\right)}\nonumber\\
    & - (1-\alpha) \E{s,a\sim \rho}{\gamma \sum_{s'} p(s'|s,a)\pi(a'|s')S(s',g)-S(s,g)}
\end{align}

For deterministic dynamics, this reduces to the action-free \scogo objective:
\begin{align}
 &\min_{S(s,a)}  \beta (1-\gamma)\E{d_0(s)}{S(s,g)} \nonumber\\
    &+\E{s,a\sim \demix}{\max\left(0, (f')^{-1} \left(y(s,a,g)\right)\right)y(s,a,g)- f\left(\max\left(0, (f')^{-1} \left(y(s,a,g)\right)\right)\right)}\nonumber\\
    & - (1-\beta) \E{s,a\sim \rho}{\gamma S(s',g)-S(s,g)}
\end{align}

where $y(s,a,g) = \gamma S(s',g)-S(s,g)$.

Note that we no longer need actions in the offline dataset to learn an optimal goal conditioned score function. This score function can be used to learn presentation in action-free datasets as well as for transfer of value function across differing action-modalities where agents share the same observation space (eg. images as observations).

\end{proof}

\section{\scogo algorithmic details}
\label{app:algo_details}
\subsection{\scogo with common $f$-divergences}

\textbf{a. KL divergence}

We consider the reverse KL divergence and start with the general \scogo objective:
 \begin{multline}
    \label{eq:scogo_sp}
        \max_{\pi_g}\min_{S}  \beta (1-\gamma)\E{d_0,\pi_g}{S(s,a,g)} +\E{s,a,g\sim \demix}{f^*(\gamma P^{\pi_g} S(s,a,g )-S(s,a,g))}  \\-(1-\beta) \E{s,a,g\sim \rho}{\gamma P^{\pi_g} S(s,a,g)-S(s,a,g)}
\end{multline}

Plugging in the conjugate $f^*$ for reverse KL divergence we get:

\begin{multline}
\label{eq:original_kl_scogo}
    \max_{\pi_g}\min_{S}  \beta (1-\gamma)\E{d_0,\pi_g}{S(s,a,g)} +\E{s,a,g\sim \demix}{e^{(\gamma P^{\pi_g} S(s,a,g )-S(s,a,g))}} \\-(1-\beta) \E{s,a,g\sim \rho}{\gamma P^{\pi_g} S(s,a,g)-S(s,a,g)}
\end{multline}

Using the telescoping sum for the last term in the objective above, we can simplify it as follows:

\begin{multline}
    \max_{\pi_g}\min_{S}  \beta (1-\gamma)\E{d_0,\pi_g}{S(s,a,g)} +\E{s,a,g\sim \demix}{e^{(\gamma P^{\pi_g} S(s,a,g )-S(s,a,g))}} \\+(1-\beta) \E{s,g\sim d_0,a\sim \rho(\cdot|s,g)}{ S(s,a,g)}
\end{multline}

With the initial state distribution $d_0$ set to the offline dataset distribution $\rho$, and  Since our initial state distribution is the same as offline data distribution, we get:

\begin{multline}
    \max_{\pi_g}\min_{S}  \beta (1-\gamma)\E{\rho,\pi_g}{S(s,a,g)} +\E{s,a,g\sim \demix}{e^{(\gamma P^{\pi_g} S(s,a,g )-S(s,a,g))}} \\+(1-\beta) \E{\rho}{ S(s,a,g)}
\end{multline}

Collecting terms together we get:

\begin{multline}
\label{eq:unstable_kl_scogo}
    \max_{\pi_g}\min_{Q}  \E{\rho}{\E{a\sim\pi}{\beta (1-\gamma) S(s,a,g)}+\E{a\sim \rho}{(1-\beta) S(s,a,g)}} \\+\E{s,a,g\sim \demix}{e^{(\gamma P^{\pi_g} S(s,a,g )-S(s,a,g))}} 
\end{multline}

The objective for \scogo with reverse KL divergence pushes down the "score" of offline dataset transitions selectively (without pushing down score of the goal-transition distribution) while minimizing the term resembling bellman regularization that also encourages increasing score at the mixture dataset jointly over the offline dataset as well as the goal transition distribution. 
% Using proposition 1 from~\cite{kim2021demodice}, we can obtain a surrogate objective for the above optimization with the same optimal solution. Moreover, this conversion avoids the training instability of Equation~\ref{eq:unstable_kl_scogo}. The surrogate objective for reverse KL is given by:

% \begin{multline}
%     \max_{\pi_g}\min_{S}  \E{\rho}{\E{a\sim\pi}{\beta (1-\gamma) S(s,a,g)}+\E{a\sim \rho}{(1-\beta) S(s,a,g)}} \\+\log \E{s,a,g\sim \demix}{e^{(\gamma P^{\pi_g} S(s,a,g )-S(s,a,g))}} 
% \end{multline}

% Rearranging terms we get:

% \begin{multline}
%     \max_{\pi_g}\max_{S}  \log e^{\E{\rho}{\E{a\sim\pi}{\beta (1-\gamma) S(s,a,g)}+\E{a\sim \rho}{(1-\beta) S(s,a,g)}}} \\+\log \E{s,a,g\sim \demix}{e^{(\gamma P^{\pi_g} S(s,a,g )-S(s,a,g))}} 
% \end{multline}

% \begin{multline}
%     \max_{\pi_g}\max_{S}  \log \frac{e^{\E{\rho}{\E{a\sim\pi}{\beta (1-\gamma) S(s,a,g)}+\E{a\sim \rho}{(1-\beta) S(s,a,g)}}}}{\E{s,a,g\sim \demix}{e^{(\gamma P^{\pi_g} S(s,a,g )-S(s,a,g))}}} 
% \end{multline}

% The objective for \scogo with reverse KL divergence bears resemblance to the infoNCE objective~\cite{} used commonly for unsupervised learning witht the distinction that our objective is pushing down the "score" of offline dataset transitions while minimizing the term resembling bellman consistency that also encourages increasing score at the mixture dataset jointly over the offline dataset as well as the goal transition distribution. 

\textbf{b. Pearson chi-squared divergence}

We consider the Pearson $\chi^2$ and start with the general \scogo objective:
 \begin{multline}
        \max_{\pi_g}\min_{S}  \beta (1-\gamma)\E{d_0,\pi_g}{S(s,a,g)} +\E{s,a,g\sim \demix}{f^*(\gamma P^{\pi_g} S(s,a,g )-S(s,a,g))}  \\-(1-\beta) \E{s,a,g\sim \rho}{\gamma P^{\pi_g} S(s,a,g)-S(s,a,g)}
\end{multline}

With the initial state distribution $d_0$ set to the offline dataset distribution $\rho$, and plugging in the conjugate $f^*$ for Pearson $\chi^2$ divergence we get:
 \begin{multline}
        \max_{\pi_g}\min_{S}  \beta (1-\gamma)\E{d_0,\pi_g}{S(s,a,g)} +0.25\E{s,a,g\sim \demix}{(\gamma P^{\pi_g} S(s,a,g )-S(s,a,g))^2}  \\+\E{s,a,g\sim \demix}{(\gamma P^{\pi_g} S(s,a,g )-S(s,a,g))}-(1-\beta) \E{s,a,g\sim \rho}{\gamma P^{\pi_g} S(s,a,g)-S(s,a,g)}
\end{multline}
Using the fact that $\demix = \beta q(s,a,g)+(1-\beta)\rho(s,a,g)$, we can further simplify the above equation to:
 \begin{multline}
        \max_{\pi_g}\min_{S}  \beta (1-\gamma)\E{d_0,\pi_g}{S(s,a,g)} +0.25\E{s,a,g\sim \demix}{(\gamma P^{\pi_g} S(s,a,g )-S(s,a,g))^2}  \\+\beta \E{s,a,g\sim q}{(\gamma P^{\pi_g} S(s,a,g )-S(s,a,g))}
\end{multline}  

Collecting terms together we get:

 \begin{multline}
        \max_{\pi_g}\min_{S}  \beta (1-\gamma)\E{\rho,\pi_g}{S(s,a,g)} +\beta \E{s,g\sim q,a\sim\pi_g} {\gamma P^{\pi_g} S(s,a,g )}\\
        -\beta \E{s,a,g\sim q}{S(s,a,g)}
        +0.25\E{s,a,g\sim \demix}{(\gamma P^{\pi_g} S(s,a,g )-S(s,a,g))^2}  
\end{multline}  

Observing the equation above, we note that the first two terms decrease score at offline data distribution as well as the goal transition distribution when actions are sampled according to the policy $\pi_g$. Simultaneously the third term pushes score up for the $\set{s,a,g}$ tuples that are sampled from goal transition distribution. Finally the last term encouraged enforces a bellman regularization enforcing smoothness is the scores of neighbouring states.

% \subsection{Psuedocode for \scogo}

\section{\scogo experimental details}
\label{ap:experimental_details}
\subsection{Tasks with observations as states}
\textbf{Environments:} For the offline GCRL experiments we consider the benchmark used in prior work GoFar  and extend it with locomotion tasks. For the manipulations tasks we consider the Fetch environment and a dextrous shadow hand environment. Fetch environments~\citep{plappert2018multi} consists of a manipulator with seven degrees of freedom along with a parallel gripper. The set of environments get a sparse reward of 1 when the goal is within 5 cm and 0 otherwise. The action space is 4 dimensional (3 dimension cartesian control + 1 dimension gripper control). The shadow hand is 24 DOF manipulator with 20-dimensional action space. The goal is 15-dimension specifying the position for each of the five fingers. The tolerance for goal reaching is 1 cm. For the locomotion environments, the task is to achieve a particular velocity in the x direction and stay at the velocity. For HalfCheetah, the target velocity is set to 11.0 and for Ant the target velocity is 5.0. For locomotion environments, the tolerance for goal reaching if 0.5. The MuJoCo environments used in this work are \href{https://github.com/deepmind/mujoco/blob/main/LICENSE}{licensed under CC BY 4.0}. 

\textbf{Offline Datasets:} We use existing datasets from the offline GCRL benchmark used in~\citep{ma2022far} for all manipulation tasks except Reacher, SawyerReach, and SawyerDoor. For Reacher, SawyerReach, and SawyerDoor we use existing datasets from~\citep{yang2022rethinking}. These datasets are comprised on x\% random data and (100-x)\% expert data depending on the coverage over goals reached in individual datasets. We create our own datasets for locomotion by using 'random/medium/medium-replay' data as our offline (suboptimal) data combined with 30 trajectories from corresponding 'expert' datasets. The datasets used from D4RL are  \href{https://github.com/Farama-Foundation/D4RL/blob/master/LICENSE}{licensed under Apache 2.0}.

\textbf{Baselines:} To benchmark and analyze the performance of our proposed methods for offline imitation learning with suboptimal data, we consider the following representative baselines in this work: GoFAR \citep{ma2022far}, WGCSL \citep{yang2022rethinking}, GCSL \citep{ghosh2019learning}, and Actionable Models \citep{chebotar2021actionable}, Contrastive RL~\citep{eysenbach2020c} and GC-IQL~\cite{kostrikov2021offline}. GoFAR is a dual occupancy matching approach to GCRL that formulates it as a weighted regression problem. WGCSL and GSCL use goal-conditioned behavior cloning with goal relabelling as the base algorithms and WGCL uses weights to learn improved policy over GCSL. Actionable models uses conservative learning with goal chaining to learn goal-reaching behaviours using offline datasets. Contrastive RL treats GCRL as a classification problem - contrastive goals that are achieved in trajectory from random goals. Finally, GC-IQL extends the single task offline RL algorithm IQL to GCRL. 

The open-source implementations of the baselines GoFAR, WGCSL, GCSL, Actionable models, Contrastive RL and IQL are provided by the authors \citep{ma2022far} and employed in our experiments. We use the hyperparameters provided by the authors, which are consistent with those used in the original GoFAR paper, for all the MuJoCo locomotion and manipulation environments. We implement contrastive learning using the code from \href{https://github.com/chongyi-zheng/stable_contrastive_rl}{Contrastive RL repository}. GC-IQL is implemented using code from author's implementation \href{https://github.com/ikostrikov/implicit_q_learning}{found here}. 

\textbf{Architecture and Hyperparameters}
For the baselines, we use tuned hyperparameters from previous works that were tuned on the same set of tasks and datasets. Implementation for \scogo shares the same network architecture as baselines. GoFAR additionally requires training a discriminator. For all
experiments, all methods are trained for 10 seeds with each training run. Fetch manipulation (except Push) tasks are trained for a maximum of 400k minibatch updates of size 512 whereas all other environments training is done for 1M minibatch updates. \reb{The expectile parameter $\tau$ was searched over [0.65 ,0.7,0.8,0.85]. For the results shown in table~\ref{table:offline-gcrl-full-discounted-return}, Fetch and Sawyer environments use $\tau=0.8$, Locomotion and Adroit hand environments use $\tau=0.7$.  In general, the HER ratio is searched over [0.2,0.5,0.8,1.0] for all methods and the best one was selected. HER ratio of 0.8 gave best performance across all tasks for \scogo. }

The architectures and hyperparameters for all methods are reported in Table~\ref{tab:smore_hp}.
\begin{table}[h!]
  \begin{center}
    \begin{tabular}{l|c}
      \toprule % <-- Toprule here
      \textbf{Hyperparameter} & \textbf{Value}\\
      \midrule % <-- Midrule here
      Policy updates $n_{pol}$ & 1\\
      Policy learning rate & 3e-4\\
      Value learning rate & 3e-4\\     
      MLP layers &  (256,256)\\
      LR decay schedule & cosine\\
      Discount factor & 0.99\\
      LR decay schedule & cosine\\
      Batch Size            & 512\\
      Mixture ratio $\beta$ & 0.5\\
      \reb{Policy temperature ($\alpha$)}  & 3.0\\
      \bottomrule % <-- Bottomrule here
    \end{tabular}
  \end{center}
  \caption{Hyperparameters for \scogo. }
      \label{tab:smore_hp}
\vspace{10pt}
\end{table}

\subsection{Tasks with observations as images}

\begin{table}[h]
\label{tab:hparams}
\begin{center}
\resizebox{0.5\textwidth}{!}{
\setlength{\tabcolsep}{4pt}
\begin{tabular}{p{5cm}|c}
\toprule
Hyperparameters & Values  \\
\midrule
batch size & 2048  \\ \midrule
number of training epochs & 300 \\ \midrule
number of training iterations per epoch & 1000  \\ \midrule
Horizon & 400 \\ \midrule
image encoder architecture & 3-layer CNN \\ \midrule
policy network architecture & (1024, 4) MLP  \\ \midrule
critic/score network architecture & (1024, 4) MLP \\ \midrule
weight initialization for final layers of critic and policy & $\textsc{Unif}[-10^{-12}, 10^{-12}]$  \\ \midrule
policy std deviation & 0.15  \\ \midrule
representation dimension & 16 \\ \midrule
data augmentation & random cropping \\ \midrule
discount factor & 0.99  \\ \midrule
learning rate  & 3e-4  \\
\bottomrule
\end{tabular}
}
\end{center}
\caption{\reb{\footnotesize Hyperparameters for image-observation GCRL from~\cite{zheng2023stabilizing}.}}
\end{table}

% \section{Experimental Details}
% \label{appendix:exp-details}
\reb{\paragraph{Tasks and dataset} Our experiments use a suite of simulated goal-conditioned control tasks based on prior work~\cite{zheng2023stabilizing}. The observations and goals are $48 \times 48 \times 3$ RGB images. The evaluation tasks require multi-trajectory stitching whereas the dataset contains trajectories solving only parts of the evaluation tasks. }

\reb{ In the simulation, we employed an offline manipulation dataset comprising near-optimal examples of basic action sequences, including the initiation of drawer movement, the displacement of blocks, and the grasping of items. The demonstrations exhibit variations in length, ranging between 50 to 100 horizon, while the offline dataset contains a total of 250,000 state transitions in its entirety. It is important to note that the offline trajectories do not depict a complete progression from the initial condition to the final objective. For the purposes of evaluation, we consider 4 tasks similar to~\cite{zheng2023stabilizing}, against which we compare the success rates in realizing these specified objectives.}

\reb{\paragraph{Baseline and \scogo implementations.} We use the open-source implementation of \href{https://github.com/chongyi-zheng/stable_contrastive_rl}{Stable-contrastive RL} to use the same design decisions and implement \scogo, GC-IQL on that codebase. We use the same hyperparameters as the stable-contrastive RL implementation for the shared hyperparameters. The hyperparameters for \scogo were kept the same as in Table~\ref{tab:smore_hp}.}

\section{Additional experiments}
\label{ap:additional_experiments}

\subsection{Results on offline GCRL benchmark with varying expert coverage in offline dataset}
We ablate the effect of dataset quality on the performance of an offline GCRL method in this sections. Table~\ref{tab:5_percent_data},~\ref{tab:2d5_percent_data},~\ref{tab:1_percent_data} show performance of all methods with 5\%, 2.5\% and 1\% expert data in the offline dataset respectively.

\begin{table}
\centering
\resizebox{0.8\textwidth}{!}{
\begin{tabular}{l|rr|r|rr}
\toprule
\multicolumn{1}{c|}{\textbf{Task}}
& \multicolumn{2}{c|}{\textbf{ Behavior cloning}}
& \multicolumn{1}{c}{\textbf{Contrastive RL}}& \multicolumn{2}{c}{\textbf{RL+sparse reward}} \\ 
  & \textbf{WGCSL} & \textbf{GCSL} & \textbf{CRL} & \textbf{AM}  & \textbf{IQL} \\ 
\midrule 
% PointReach  && $\pm${\scriptsize } & & $\pm${\scriptsize }  &  $\pm${\scriptsize }&  \pm${\scriptsize }\\ 
% PointRooms  && $\pm${\scriptsize } & & $\pm${\scriptsize }  &  $\pm${\scriptsize }& $\pm${\scriptsize } \\ 
Reacher  &15.30 $\pm${\scriptsize0.58 }& 14.01 $\pm${\scriptsize 0.36}&16.62 $\pm${\scriptsize 2.09} & 23.68$\pm${\scriptsize 0.58}&  8.86 $\pm$ {\scriptsize0.61} \\ 
SawyerReach  &14.06 $\pm${\scriptsize 0.08 } & 12.05$\pm${\scriptsize1.23} & 23.03$\pm${\scriptsize 1.17} &  23.37$\pm${\scriptsize 2.29}& 36.19  $\pm$ {\scriptsize 0.01}\\ 
SawyerDoor  & 16.79$\pm${\scriptsize 0.75} & 18.29$\pm${\scriptsize0.94} & 12.26 $\pm${\scriptsize 3.94} & 16.63 $\pm${\scriptsize 0.76}&   29.31$\pm$ {\scriptsize0.88}\\
% FetchReach   & $\pm$ {\scriptsize } &  $\pm$ {\scriptsize}   &  $\pm$ {\scriptsize} &  $\pm$ {\scriptsize} &  31.42 $\pm$ {\scriptsize1.02} \\ 
FetchPick   &  6.87$\pm$ {\scriptsize 0.77} & 6.54 $\pm$ {\scriptsize1.85}  &  0.21$\pm$ {\scriptsize0.29}  & 0.45 $\pm$ {\scriptsize0.32} &   15.24$\pm$ {\scriptsize1.27} \\
FetchPush  &  10.62$\pm$ {\scriptsize0.98 } & 12.38 $\pm$ {\scriptsize1.10}   & 3.60 $\pm$ {\scriptsize0.59}  & 2.74 $\pm$ {\scriptsize0.70} &  19.95 $\pm$ {\scriptsize1.94} \\
FetchSlide  &  2.62$\pm$ {\scriptsize 1.15} & 2.03 $\pm$ {\scriptsize0.01} &  0.41$\pm$ {\scriptsize0.03}  & 0.31 $\pm$ {\scriptsize}0.31 & 3.25  $\pm$ {\scriptsize1.02}  \\ 
% HandReach  &\textbf{} $\pm$ {\scriptsize }&  $\pm$ {\scriptsize} &  $\pm$ {\scriptsize } &  $\pm$ {\scriptsize} &  $\pm$ {\scriptsize} &   $\pm$ {\scriptsize}&   $\pm$ {\scriptsize} \\ 
\midrule
\bottomrule
\end{tabular}
}
\caption{Discounted Return for the offline GCRL benchmark with 5\% expert data. Results are averaged over 10 seeds.}
\label{tab:5_percent_data}
\end{table}

\begin{table}
\centering
\resizebox{0.8\textwidth}{!}{
\begin{tabular}{l|rr|r|rr}
\toprule
\multicolumn{1}{c|}{\textbf{Task}}
& \multicolumn{2}{c|}{\textbf{ Behavior cloning}}
& \multicolumn{1}{c}{\textbf{Contrastive RL}}& \multicolumn{2}{c}{\textbf{RL+sparse reward}} \\ 
 & \textbf{WGCSL} & \textbf{GCSL} & \textbf{CRL} & \textbf{AM}  & \textbf{IQL} \\ 
\midrule 
% PointReach  && $\pm${\scriptsize } & & $\pm${\scriptsize }  &  $\pm${\scriptsize }&  \pm${\scriptsize }\\ 
% PointRooms  && $\pm${\scriptsize } & & $\pm${\scriptsize }  &  $\pm${\scriptsize }& $\pm${\scriptsize } \\ 
Reacher  & 13.03$\pm${\scriptsize 0.56}& 12.17 $\pm${\scriptsize 0.8}&19.63 $\pm${\scriptsize 3.09} & 24.78$\pm${\scriptsize0.23 }&   4.44$\pm$ {\scriptsize0.70} \\ 
SawyerReach   & 11.455$\pm${\scriptsize 1.37} & 11.34$\pm${\scriptsize1.18} &25.35 $\pm${\scriptsize 0.8} &  25.19$\pm${\scriptsize 0.61}&  35.73 $\pm$ {\scriptsize0.22}\\ 
SawyerDoor & 16.79$\pm${\scriptsize 0.29 } & 13.20$\pm${\scriptsize0.53} & 14.78 $\pm${\scriptsize 5.29} &  16.59 $\pm${\scriptsize 1.39 }&  16.87 $\pm$ {\scriptsize4.21}\\
% FetchReach & $\pm$ {\scriptsize } &  $\pm$ {\scriptsize}   &  $\pm$ {\scriptsize} &  $\pm$ {\scriptsize} &   30.93$\pm$ {\scriptsize1.98} \\ 
FetchPick   &  4.39$\pm$ {\scriptsize1.35 } & 4.99 $\pm$ {\scriptsize0.11}  &  0.21$\pm$ {\scriptsize0.29}  &  0.24$\pm$ {\scriptsize0.27} &  11.79 $\pm$ {\scriptsize1.78} \\
FetchPush   &  8.01$\pm$ {\scriptsize 1.96} &  8.04$\pm$ {\scriptsize0.34}   &  3.60$\pm$ {\scriptsize0.59}  &  2.02$\pm$ {\scriptsize0.48} &  19.66 $\pm$ {\scriptsize1.69} \\
FetchSlide   &  2.33$\pm$ {\scriptsize0.23 } &  2.37 $\pm$ {\scriptsize0.83} & 0.44 $\pm$ {\scriptsize0.016}  &  0.45$\pm$ {\scriptsize0.44} &   1.83$\pm$ {\scriptsize1.31}  \\ 
% HandReach  &\textbf{} $\pm$ {\scriptsize }&  $\pm$ {\scriptsize} &  $\pm$ {\scriptsize } &  $\pm$ {\scriptsize} &  $\pm$ {\scriptsize} &   $\pm$ {\scriptsize}&   $\pm$ {\scriptsize} \\ 
\midrule
\bottomrule
\end{tabular}
}
\caption{Discounted Return for the offline GCRL benchmark with 2.5\% expert data. Results are averaged over 10 seeds.}
\label{tab:2d5_percent_data}
\end{table}

\begin{table}
\centering
\resizebox{0.8\textwidth}{!}{
\begin{tabular}{l|rr|r|rr}
\toprule
\multicolumn{1}{c|}{\textbf{Task}}
& \multicolumn{2}{c|}{\textbf{ Behavior cloning}}
& \multicolumn{1}{c}{\textbf{Contrastive RL}}& \multicolumn{2}{c}{\textbf{RL+sparse reward}} \\ 
 & \textbf{WGCSL} & \textbf{GCSL} & \textbf{CRL} & \textbf{AM}  & \textbf{IQL} \\ 
\midrule 
% PointReach  && $\pm${\scriptsize } & & $\pm${\scriptsize }  &  $\pm${\scriptsize }&  \pm${\scriptsize }\\ 
% PointRooms  && $\pm${\scriptsize } & & $\pm${\scriptsize }  &  $\pm${\scriptsize }& $\pm${\scriptsize } \\ 
Reacher   & 13.56$\pm${\scriptsize 0.69}& 12.27 $\pm${\scriptsize 1.45}& 17.94$\pm${\scriptsize 3.71} & 24.89$\pm${\scriptsize 0.34}&  4.28 $\pm$ {\scriptsize0.92} \\ 
SawyerReach   &10.71 $\pm${\scriptsize 0.69} & 11.79$\pm${\scriptsize1.46} & 25.61$\pm${\scriptsize 0.39} &  25.54$\pm${\scriptsize 0.95}&  31.31  $\pm$ {\scriptsize2.08 }\\ 
SawyerDoor  &15.18 $\pm${\scriptsize 0.81} &11.89$\pm${\scriptsize1.51} &  10.26$\pm${\scriptsize 4.61} &  18.04$\pm${\scriptsize 1.8}& 17.11  $\pm$ {\scriptsize4.45}\\
% FetchReach  & $\pm$ {\scriptsize } &  $\pm$ {\scriptsize}   &  $\pm$ {\scriptsize} &  $\pm$ {\scriptsize} &  25.41 $\pm$ {\scriptsize 0.75} \\ 
FetchPick   & 1.89 $\pm$ {\scriptsize1.22 } & 3.30 $\pm$ {\scriptsize0.66}  & 0.42 $\pm$ {\scriptsize0.29}  &  0.41 $\pm$ {\scriptsize0.22} &  7.90 $\pm$ {\scriptsize1.22} \\
FetchPush & 6.44 $\pm$ {\scriptsize 3.64} & 6.43 $\pm$ {\scriptsize0.56}   & 1.69 $\pm$ {\scriptsize1.56}  &  2.63$\pm$ {\scriptsize3.04} &  7.11 $\pm$ {\scriptsize2.60} \\
FetchSlide & 1.77 $\pm$ {\scriptsize 0.24} &  1.11$\pm$ {\scriptsize0.26} & 0.0 $\pm$ {\scriptsize0.0}  & 0.10 $\pm$ {\scriptsize0.11} &  0.80 $\pm$ {\scriptsize0.48}  \\ 
% HandReach  &\textbf{} $\pm$ {\scriptsize }&  $\pm$ {\scriptsize} &  $\pm$ {\scriptsize } &  $\pm$ {\scriptsize} &  $\pm$ {\scriptsize} &   $\pm$ {\scriptsize}&   $\pm$ {\scriptsize} \\ 
\midrule
\bottomrule
\end{tabular}
}
\caption{Discounted Return for the offline GCRL benchmark with 1\% expert data. Results are averaged over 10 seeds.}
\label{tab:1_percent_data}
\end{table}

\subsection{Success Rate and Final distance to goal on Manipulation tasks}
Table~\ref{tab:success_rate} and Table~\ref{tab:final_distance_to_goal} reports the success rate and final distance to goal metrics on manipulation tasks.
%  \subsection{Learning curves for \scogo}

% Figure~\ref{fig:offline_gcrl_learning_curves} shows learning curves for all the \scogo experiments in domains where observations are states.

% \begin{figure}[h]
% \vspace{-5pt}
% \begin{center}
%           \includegraphics[width=1.00\linewidth]{figures/learning_curves/all_training_plots.pdf}
% \end{center}
% \vspace{-1.0mm}
% \caption{ Learning curves with all state-based tasks using \scogo corresponding to Table~\ref{table:offline-gcrl-full-discounted-return}}
% \label{fig:offline_gcrl_learning_curves}
% \end{figure}

\subsection{Robustness of mixture distribution parameter $\beta$}

We find that \scogo is quite robust to the mixture distribution parameter $\beta$ except in the environment FetchPush where $\beta=0.5$ is the most performant. Table~\ref{table:robustness_mixture} shows this result empirically.

\subsection{How much does HER contribute to the performance improvements of SMORE?}

\reb{GoFAR demonstrated improved performance without relying on HER. The authors also demonstrated that HER is detrimental to GoFAR's performance. In this section, we aim to conduct a similar stufy and see how much HER contributed to \scogo's performance. Table~\ref{table:smore_without HER} shows that HER gives \scogo a small performance boost and show that \scogo is still able to outperform GoFAR without HER.}

\subsection{Comparison to variants of GoFAR}

\reb{GoFAR formulates GCRL as an occupancy matching problem, but it is also suggested that using a discriminator is optional. Without a discriminator, GoFAR reduces to a sparse reward RL problem. Table~\ref{table:gofar_without_reward} shows that GoFAR achieves poor performance when a reward function is substituted in place on an discriminator. We also study if the performance benefits we obtain are due to the offline learning strategy we used from IQL. We modify GoFAR with discriminator reward to use expectile loss for value learning and AWR for policy learning. Results in Table~\ref{table:gofar_without_reward} shows that no performance gains were observed.}

\subsection{Offline GCRL with purely suboptimal data}

\reb{In this experiment, we study offline GCRL from purely suboptimal datasets. Except FetchReach, these datasets provide very sparse coverage of goals expected to reach in evaluation. Table~\ref{table:gofar_without_expert_data} shows the robustness of \scogo even in the setting of poor quality offline data.}

\subsection{Comparison with in-sample learning methods}

\reb{In-sample learning methods perform value improvement using bellman backups without OOD action sampling. This makes them a particularly suitable candidate for offline setting. We compare against a number of recent in-sample learning methods, IQL~\citep{kostrikov2021offline}, SQL/EQL~\cite{xu2023offline} and XQL~\citep{garg2023extreme}. Table~\ref{table:smore_in_sample_rl} compares \scogo to in-sample learning methods adapted to GCRL.}

\subsection{Ablating components of \scogo for offline setting}

\reb{In offline setting, it is well known that bellman backups suffer from overestimation and results in poor policy performance. We validate the utility of the components used in this work in Table~\ref{table:smore_without_practical_changes} : expectile loss function and constrained policy optimization with AWR.}

\begin{table}
\centering
\resizebox{0.6\textwidth}{!}{
\begin{tabular}{l|c|c|c|c}
\toprule
\multicolumn{1}{c|}{\textbf{Task}}
& \multicolumn{1}{c|}{\textbf{$\beta=0.5$}}
& \multicolumn{1}{c|}{\textbf{$\beta=0.7$}} & \multicolumn{1}{c}{\textbf{$\beta=0.8$}}& \multicolumn{1}{c}{\textbf{$\beta=0.9$}} \\ 
\midrule  
FetchReach &35.08 $\pm$ {\scriptsize0.54}&36.57 $\pm$ {\scriptsize0.20} &36.59$\pm$ {\scriptsize0.30}&  36.30$\pm$ {\scriptsize0.30}\\ 
FetchPick   &26.47$\pm$ {\scriptsize0.34}&27.04$\pm$ {\scriptsize0.81}&27.43$\pm$ {\scriptsize0.97}&  27.89 $\pm$ {\scriptsize1.19}\\
FetchPush&26.83 $\pm$ {\scriptsize1.21}&16.20$\pm$ {\scriptsize1.11}&11.50$\pm$ {\scriptsize1.19} & 13.85$\pm$ {\scriptsize5.53} \\
FetchSlide&4.99$\pm$ {\scriptsize0.40}&3.76$\pm$ {\scriptsize0.75}&3.43$\pm$ {\scriptsize2.4}&4.10$\pm$ {\scriptsize1.20}\\ 
\midrule
\bottomrule
\end{tabular}
}
\caption{Discounted Return for the offline GCRL benchmark with varying mixture coefficients in offline dataset. Results are averaged over 10 seeds.}
\label{table:robustness_mixture}
\end{table}

\begin{table}
\resizebox{\textwidth}{!}{
\begin{tabular}{l|rr|rr|r|rr}
\toprule
\multicolumn{1}{c|}{\textbf{Task}}
&\multicolumn{2}{c|}{\textbf{Occupancy Matching}}& \multicolumn{2}{c|}{\textbf{Behavior cloning}}& \multicolumn{1}{c|}{\textbf{Contrastive RL}}
& \multicolumn{2}{c}{\textbf{RL+sparse reward}} \\ 
 & \textbf{SMORe} & \textbf{GoFAR}  & \textbf{WGCSL} & \textbf{GCSL} & \textbf{CRL}& \textbf{AM}  & \textbf{IQL} \\ 
\midrule 
% PointReach  && $\pm${\scriptsize } & & $\pm${\scriptsize }  &  $\pm${\scriptsize }&  \pm${\scriptsize }\\ 
% PointRooms  && $\pm${\scriptsize } & & $\pm${\scriptsize }  &  $\pm${\scriptsize }& $\pm${\scriptsize } \\ 
Reacher  & 0.875$\pm${\scriptsize 0.07}& 0.90$\pm${\scriptsize 0.01} & 0.97$\pm${\scriptsize0.014 }& 0.92 $\pm${\scriptsize 0.08}& 0.76$\pm${\scriptsize 0.74} & 1.0$\pm${\scriptsize 0.1} &  0.26 $\pm$ {\scriptsize0.06}\\ 
SawyerReach  &0.98$\pm${\scriptsize 0.014}& 0.75$\pm${\scriptsize0.04} & 1.0$\pm${\scriptsize 0.0} & 0.98$\pm${\scriptsize0.02} & 0.98$\pm${\scriptsize 0.018} &  1.0$\pm${\scriptsize 0.1}&  0.81 $\pm$ {\scriptsize0.01}\\ 
SawyerDoor  &0.875$\pm${\scriptsize 0.038}& 0.5$\pm${\scriptsize0.12} &0.78 $\pm${\scriptsize 0.10} & 0.5$\pm${\scriptsize 0. 12} &  0.22$\pm${\scriptsize 0.11} &  0.3$\pm${\scriptsize 0.11}&  0.84 $\pm$ {\scriptsize0.06}\\
FetchReach  &1.0$\pm$ {\scriptsize 0.0}& 1.0 $\pm$ {\scriptsize0.0} & 1.0$\pm$ {\scriptsize 0.0} & 0.98 $\pm$ {\scriptsize0.05}   &  1.0$\pm$ {\scriptsize 0.0} &  1.0$\pm$ {\scriptsize1.0} &  1.0 $\pm$ {\scriptsize0.0} \\ 
FetchPick  &0.925 $\pm$ {\scriptsize 0.045}& 0.84 $\pm$ {\scriptsize0.09} &  0.54$\pm$ {\scriptsize 0.16} & 0.54 $\pm$ {\scriptsize0.20}  & 0.42 $\pm$ {\scriptsize0.29}  & 0.78 $\pm$ {\scriptsize0.15}  &  0.86 $\pm$ {\scriptsize0.11}\\
FetchPush  &0.90$\pm$ {\scriptsize 0.07}&  0.88$\pm$ {\scriptsize0.09} &  0.76$\pm$ {\scriptsize 0.12} & 0.72 $\pm$ {\scriptsize 0.15}   &  0.06$\pm$ {\scriptsize0.03}  &  0.67$\pm$ {\scriptsize0.14}  &  0.65 $\pm$ {\scriptsize0.052}\\
FetchSlide  &0.315$\pm$ {\scriptsize 0.07}& 0.18 $\pm$ {\scriptsize0.12} &  0.18$\pm$ {\scriptsize 0.14} &  0.17$\pm$ {\scriptsize0.13} &0.0  $\pm$ {\scriptsize0.0}  &  0.11$\pm$ {\scriptsize0.09}  &   0.26$\pm$ {\scriptsize0.057} \\ 
HandReach & 0.47$\pm$ {\scriptsize 0.11}& 0.40 $\pm$ {\scriptsize}0.20 &  0.25$\pm$ {\scriptsize 0.23} &  0.047$\pm$ {\scriptsize0.10} &  0.0$\pm$ {\scriptsize0.0} &  0.0 $\pm$ {\scriptsize0.0} &  0.0 $\pm$ {\scriptsize0.0}\\ 
\midrule
\bottomrule
\end{tabular}
}
\caption{Success Rate for the offline GCRL benchmark with 10\% expert data. Results are averaged over 10 seeds.}
\label{tab:success_rate}
\end{table}

\begin{table}
\resizebox{\textwidth}{!}{
\begin{tabular}{l|rr|rr|r|rr}
\toprule
\multicolumn{1}{c|}{\textbf{Task}}
&\multicolumn{2}{c|}{\textbf{Occupancy Matching}}& \multicolumn{2}{c|}{\textbf{Behavior cloning}}& \multicolumn{1}{c|}{\textbf{Contrastive RL}}
& \multicolumn{2}{c}{\textbf{RL+sparse reward}} \\ 
 & \textbf{SMORe} & \textbf{GoFAR}  & \textbf{WGCSL} & \textbf{GCSL} & \textbf{CRL}& \textbf{AM}  & \textbf{IQL} \\ 
\midrule 
% PointReach  && $\pm${\scriptsize } & & $\pm${\scriptsize }  &  $\pm${\scriptsize }&  \pm${\scriptsize }\\ 
% PointRooms  && $\pm${\scriptsize } & & $\pm${\scriptsize }  &  $\pm${\scriptsize }& $\pm${\scriptsize } \\ 
Reacher  & 0.02$\pm${\scriptsize 0.01}& 0.03$\pm${\scriptsize 0.01} & 0.011$\pm${\scriptsize0.01 }& 0.016 $\pm${\scriptsize 0.00}& 0.05$\pm${\scriptsize 0.03} & 0.013$\pm${\scriptsize 0.00} &   0.12$\pm$ {\scriptsize0.005}\\ 
SawyerReach  &0.008$\pm${\scriptsize 0.004}& 0.04$\pm${\scriptsize0.00} & 0.004$\pm${\scriptsize 0.00} & 0.00$\pm${\scriptsize 0.00} & 0.01$\pm${\scriptsize 0.01} & 0.01 $\pm${\scriptsize0.00 }& 0.053  $\pm$ {\scriptsize0.004}\\ 
SawyerDoor  &0.02$\pm${\scriptsize 0.029}& 0.18$\pm${\scriptsize0.00} & 0.011$\pm${\scriptsize0.00 } & 0.017$\pm${\scriptsize0.01} &  0.14$\pm${\scriptsize 0.07} & 0.06 $\pm${\scriptsize0.01 }&  0.019 $\pm$ {\scriptsize0.01}\\
FetchReach  &0.004$\pm$ {\scriptsize 0.0012}&  0.018$\pm$ {\scriptsize0.003} & 0.007$\pm$ {\scriptsize 0.0043} & 0.008 $\pm$ {\scriptsize0.008}   & 0.007 $\pm$ {\scriptsize0.001} & 0.007 $\pm$ {\scriptsize0.001} &   0.002$\pm$ {\scriptsize0.001} \\ 
FetchPick  & 0.04$\pm$ {\scriptsize0.018 }& 0.036 $\pm$ {\scriptsize0.013} &  0.094$\pm$ {\scriptsize 0.043} &  0.108$\pm$ {\scriptsize0.06}  & 0.25 $\pm$ {\scriptsize0.025}  &  0.04$\pm$ {\scriptsize0.02}  &   0.04$\pm$ {\scriptsize0.012}\\
FetchPush  &0.03$\pm$ {\scriptsize 0.003}&  0.033$\pm$ {\scriptsize0.008} &  0.041$\pm$ {\scriptsize 0.02} &  0.042$\pm$ {\scriptsize0.018}   &  0.15$\pm$ {\scriptsize0.036}  &  0.07$\pm$0.039 {\scriptsize}  &   0.05$\pm$ {\scriptsize0.006}\\
FetchSlide &0.09$\pm$ {\scriptsize0.012 }& 0.12 $\pm$ {\scriptsize0.02} &  0.173$\pm$ {\scriptsize 0.04} &  0.204$\pm$ {\scriptsize0.051} &  0.42$\pm$ {\scriptsize0.05}  &  0.198$\pm$ {\scriptsize0.059}  &   0.09$\pm$ {\scriptsize0.013} \\ 
HandReach & 0.039$\pm$ {\scriptsize 0.0108}&  0.024$\pm$ {\scriptsize0.009} & 0.035 $\pm$ {\scriptsize 0.012} &  0.038$\pm$ {\scriptsize0.013} &  0.04 $\pm$ {\scriptsize0.005} &  0.037 $\pm$0.004 {\scriptsize} &  0.08 $\pm$ {\scriptsize0.005}\\ 
\midrule
\bottomrule
\end{tabular}
}
\caption{Final distance to goal for the offline GCRL benchmark with 10\% expert data. Results are averaged over 10 seeds.}
\label{tab:final_distance_to_goal}
\end{table}

\begin{table}
\centering
\resizebox{0.6\textwidth}{!}{
\begin{tabular}{l|c|c|c}
\toprule
\multicolumn{1}{c|}{\textbf{Task}}
& SMORe
& SMORe w/o HER & GoFAR\\ 
\midrule  
FetchReach &35.08 $\pm$ {\scriptsize0.54}&34.86$\pm$ {\scriptsize1.03} &28.2$\pm${\scriptsize0.61}\\ 
SawyerReach&37.67$\pm$ {\scriptsize0.12}&37.34$\pm$ {\scriptsize0.36}&15.34$\pm${\scriptsize0.64}\\
SawyerDoor&31.48$\pm$ {\scriptsize0.46}&31.53$\pm$ {\scriptsize0.62} &18.94$\pm${\scriptsize0.01}\\
FetchPick   &26.47$\pm$ {\scriptsize0.34}&25.72$\pm$ {\scriptsize3.88}&19.7$\pm${\scriptsize2.57}\\
FetchPush&26.83 $\pm$ {\scriptsize1.21}&25.62$\pm$ {\scriptsize1.67}&18.2$\pm${\scriptsize3.00}\\
FetchSlide&4.99$\pm$ {\scriptsize0.40}&4.09$\pm$ {\scriptsize0.33}&2.47$\pm${\scriptsize1.44}\\
\midrule
\bottomrule
\end{tabular}
}
\caption{\reb{Performance gains using HER (resampling ratio=0.8) on \scogo. All other hyperparameters are kept the same between SMORe and SMORe w/o HER.}}
\label{table:smore_without HER}
\end{table}

\begin{table}
\centering
\resizebox{0.95\textwidth}{!}{
\begin{tabular}{l|c|c|c|c|c}
\toprule
\multicolumn{1}{c|}{\textbf{Task}}
& SMORe
& GoFAR (discriminator) & GoFAR (sparse reward) & r-GoFAR (sparse reward) & GoFAR (expectile loss+AWR) \\ 
\midrule  
FetchReach &35.08 $\pm$ {\scriptsize0.54}&28.2$\pm${\scriptsize0.61} &26.1$\pm${\scriptsize1.14}&0.30$\pm${\scriptsize0.43}&26.90$\pm${\scriptsize0.41}\\ 
SawyerReach&37.67$\pm$ {\scriptsize0.12}&15.34$\pm${\scriptsize0.64}&---&0.34$\pm${\scriptsize0.33}&16.17$\pm${\scriptsize3.02}\\
SawyerDoor&31.48$\pm$ {\scriptsize0.46}&18.94$\pm${\scriptsize0.01}&---&10.36$\pm${\scriptsize3.27}&22.47$\pm${\scriptsize1.13}\\
FetchPick   &26.47$\pm$ {\scriptsize0.34}&19.7$\pm${\scriptsize2.57}&17.4$\pm${\scriptsize1.78}&0.25$\pm$ {\scriptsize0.02}&18.46$\pm$ {\scriptsize2.72}\\
FetchPush&26.83 $\pm$ {\scriptsize1.21}&18.2$\pm${\scriptsize3.00}&17.4$\pm${\scriptsize2.67}&4.23$\pm${\scriptsize3.96}&17.39$\pm${\scriptsize5.44}\\
FetchSlide&4.99$\pm$ {\scriptsize0.40}&2.47$\pm${\scriptsize1.44}&5.13$\pm${\scriptsize4.05}&0.29$\pm${\scriptsize0.03}&3.59$\pm${\scriptsize2.30}\\
\midrule
\bottomrule
\end{tabular}
}
\caption{\reb{Ablation comparison with GoFAR: a) a sparse binary reward is used in place of a learned discriminator in GoFAR b) The policy and value update is replaced by AWR and expectile loss respectively. The main difference with SMORe remains the use of discriminator for training. --- indicates these environments were not considered in~\citep{ma2022far}. Our reproduced results (denoted by r-) with the official code base for binary results did not match up to the reported results.}}
\label{table:gofar_without_reward}
\end{table}

\begin{table}
\centering
\resizebox{0.55\textwidth}{!}{
\begin{tabular}{l|c|c|c|c}
\toprule
\multicolumn{1}{c|}{\textbf{Task}}
& SMORe
& GoFAR  & WGCSL & GC-IQL \\ 
\midrule  
FetchReach & 35.01$\pm$ {\scriptsize0.47}&27.37$\pm${\scriptsize0.4} &21.65$\pm${\scriptsize0.61}&23.72$\pm${\scriptsize1.18}\\ 
SawyerReach   &36.26$\pm$ {\scriptsize0.93}&5.89$\pm${\scriptsize1.36}&7.27$\pm${\scriptsize1.14}&33.08$\pm${\scriptsize0.81}\\
SawyerDoor   &20.28$\pm$ {\scriptsize2.65}&15.33$\pm${\scriptsize1.30}&13.81$\pm${\scriptsize2.72}&16.05$\pm${\scriptsize4.97}\\
FetchPick   &0.61$\pm$ {\scriptsize0.5}&0.0$\pm${\scriptsize0.0}&0.0$\pm${\scriptsize0.0}&1.31$\pm${\scriptsize1.86}\\
FetchPush& 6.39$\pm$ {\scriptsize0.68}&4.23$\pm${\scriptsize3.96}&4.27$\pm${\scriptsize3.9}&2.63$\pm${\scriptsize1.68}\\
FetchSlide&0.42$\pm$0.01 {\scriptsize}&0.059$\pm${\scriptsize0.08}&0.93$\pm${\scriptsize0.69}&0.75$\pm${\scriptsize0.58}\\
\midrule
\bottomrule
\end{tabular}
}
\caption{\reb{Discounted Return for the offline GCRL benchmark with 0\% expert data. Results are averaged over 10 seeds.}}
\label{table:gofar_without_expert_data}
\end{table}

\begin{table}
\centering
\resizebox{0.7\textwidth}{!}{
\begin{tabular}{l|c|c|c|c|c}
\toprule
\multicolumn{1}{c|}{\textbf{Task}}
& SMORe
& GC-IQL & GC-SQL & GC-EQL & GC-XQL\\ 
\midrule  
FetchReach &35.08 $\pm$ {\scriptsize0.54}&34.43$\pm$ {\scriptsize1.00} &35.67$\pm${\scriptsize0.70}&29.23$\pm${\scriptsize0.2}&33.94$\pm${\scriptsize0.49}\\ 
SawyerReach&37.67$\pm$ {\scriptsize0.12}&35.18$\pm$ {\scriptsize0.29}&37.10$\pm${\scriptsize0.24}&30.19$\pm${\scriptsize1.66}&32.88$\pm${\scriptsize2.85}\\
SawyerDoor&31.48$\pm$ {\scriptsize0.46}&25.52$\pm$ {\scriptsize1.45} &27.96$\pm${\scriptsize0.45}&3.57$\pm${\scriptsize3.51}&5.85$\pm${\scriptsize4.21}\\
FetchPick   &26.47$\pm$ {\scriptsize0.34}&16.8$\pm$ {\scriptsize3.10}&18.35$\pm${\scriptsize6.67}&1.31$\pm${\scriptsize1.86}&1.31$\pm${\scriptsize1.82}\\
FetchPush&26.83 $\pm$ {\scriptsize1.21}&22.40$\pm$ {\scriptsize0.74}&17.19$\pm${\scriptsize2.56}&2.64$\pm${\scriptsize1.30}&3.79$\pm${\scriptsize0.21}\\
FetchSlide&4.99$\pm$ {\scriptsize0.40}&4.80$\pm$ {\scriptsize1.59}&4.68$\pm${\scriptsize3.32}&0.06$\pm${\scriptsize0.08}&0.36$\pm${\scriptsize0.52}\\
\midrule
\bottomrule
\end{tabular}
}
\caption{\reb{Comparison of SMORe to in-sample RL methods - IQL~\citep{kostrikov2021offline},SQL/EQL~\citep{xu2023offline}, XQL~\citep{garg2023extreme} that learn from sparse rewards. --- in EQL denotes the learning diverged. We observed IQL to be the most stable alternative compared to SQL, EQL and XQL. SQL, EQL and XQL were implemented using author's official codebase}}
\label{table:smore_in_sample_rl}
\end{table}

\begin{table}
\centering
\resizebox{0.75\textwidth}{!}{
\begin{tabular}{l|c|c|c}
\toprule
\multicolumn{1}{c|}{\textbf{Task}}
& SMORe
& SMORe (w/o AWR) & SMORe (w/o AWR and Expectile loss)\\ 
\midrule  
FetchReach &35.08 $\pm$ {\scriptsize0.54}&0.30$\pm${\scriptsize0.29}&0.10$\pm${\scriptsize0.13}\\   
SawyerReach &36.26 $\pm$ {\scriptsize0.93}&29.31$\pm${\scriptsize0.53}&29.64$\pm${\scriptsize0.62}\\  
SawyerDoor &20.28 $\pm$ {\scriptsize2.65}&5.06$\pm${\scriptsize0.52}&2.11$\pm${\scriptsize1.59}\\ 
FetchPick   &26.47$\pm$ {\scriptsize0.34}& 1.79$\pm${\scriptsize0.65}&1.77$\pm${\scriptsize1.51}\\
FetchPush&26.83 $\pm$ {\scriptsize1.21}&4.60$\pm${\scriptsize2.51}&2.69$\pm${\scriptsize1.01}\\
FetchSlide&4.99$\pm$ {\scriptsize0.40}&0.22$\pm${\scriptsize0.33}&0.50$\pm${\scriptsize0.02}\\
\midrule
\bottomrule
\end{tabular}
}
\caption{\reb{Ablating practical components of SMORe. Without adapting for offline setting we consider in this work by using in-sample maximization or constrained policy optimization using AWR the performance degrades as expected. Without in-sample-maximization value function explodes in the offline setting and using policy that maximizes Q can often select OOD actions leading to poor performance.}}
\label{table:smore_without_practical_changes}
\end{table}

% \begin{table}
% \centering
% \resizebox{0.75\textwidth}{!}{
% \begin{tabular}{l|c|c|c|c}
% \toprule
% \multicolumn{1}{c|}{\textbf{Task}}
% & SMORe
% & SMORe (w/o AWR)& SMORe (w/o Expectile Loss) & SMORe (w/o AWR and Expectile loss)\\ 
% \midrule  
% FetchReach &35.08 $\pm$ {\scriptsize0.54}&0.30$\pm${\scriptsize0.29}&&0.10$\pm${\scriptsize0.13}\\   
% SawyerReach &36.26 $\pm$ {\scriptsize0.93}&29.31$\pm${\scriptsize0.53}&32.82$\pm${\scriptsize0.46}&29.64$\pm${\scriptsize0.62}\\  
% SawyerDoor &20.28 $\pm$ {\scriptsize2.65}&5.06$\pm${\scriptsize0.52}&12.15$\pm${\scriptsize4.07}&2.11$\pm${\scriptsize1.59}\\ 
% FetchPick   &26.47$\pm$ {\scriptsize0.34}& 1.79$\pm${\scriptsize0.65}&6.86$\pm${\scriptsize1.01}&1.77$\pm${\scriptsize1.51}\\
% FetchPush&26.83 $\pm$ {\scriptsize1.21}&4.60$\pm${\scriptsize2.51}&&2.69$\pm${\scriptsize1.01}\\
% FetchSlide&4.99$\pm$ {\scriptsize0.40}&0.22$\pm${\scriptsize0.33}&&0.50$\pm${\scriptsize0.02}\\
% \midrule
% \bottomrule
% \end{tabular}
% }
% \caption{\reb{Ablating practical components of SMORe. Without adapting for offline setting we consider in this work by using in-sample maximization or constrained policy optimization using AWR the performance degrades as expected. Without in-sample-maximization value function explodes in the offline setting and using policy that maximizes Q can often select OOD actions leading to poor performance.}}
% \label{table:smore_without_practical_changes}
% \end{table}

\end{document}